\documentclass[runningheads]{llncs}
\usepackage[T1]{fontenc}
\usepackage{cite} 
\usepackage{caption}
\usepackage{subcaption}
\usepackage{amsmath,amssymb,amsfonts}
\usepackage{algpseudocode}
\usepackage[hidelinks]{hyperref}
\usepackage{graphicx}
\usepackage{textcomp}
\usepackage{xcolor}
\usepackage{lettrine}
\usepackage{comment}
\usepackage{wrapfig}
\usepackage{booktabs, siunitx, multirow, array}
\usepackage[ruled, linesnumbered, vlined]{algorithm2e}
\usepackage[percent]{overpic}

\newif\ifwafr
\wafrfalse   

\newcommand{\ignore}[1]{}

 \def\U{\mathcal{U}}

\def\X{\mathcal{X}}

 \def\dN{\mathbb{N}} 
\def\dR{\mathbb{R}}

\def\epsilon{\varepsilon}

\newif\ifshowcomments
\showcommentstrue          

\ifshowcomments
  \newcommand{\kiril}[1]{\textcolor{red}{(\textbf{Kiril:} #1)}}
  \newcommand{\avishav}[1]{\textcolor{orange}{(\textbf{Avishav:} #1)}}
  \newcommand{\omer}[1]{\textcolor{blue}{(\textbf{Omer:} #1)}}
  \newcommand{\AZ}[1]{\textcolor{purple}{\textbf{Andrey:} #1}}
  \newcommand{\OS}[1]{\textcolor{green}{\textbf{Oren:} #1}}
  \newcommand{\todo}[1]{\textcolor{cyan}{(\textbf{TODO:} #1)}}
\else
  \newcommand{\kiril}[1]{\ignorespaces}
  \newcommand{\avishav}[1]{\ignorespaces}
  \newcommand{\omer}[1]{\ignorespaces}
  \newcommand{\AZ}[1]{\ignorespaces}
  \newcommand{\OS}[1]{\ignorespaces}
  \newcommand{\todo}[1]{\ignorespaces}
\fi

\def\niceparagraph#1{\vspace{5pt} \noindent \textbf{#1}}

\spnewtheorem{problem}{Problem}{\bfseries}{\itshape}

\newcommand{\algd}{\textsc{GTNS}\xspace}
\newcommand{\algdfull}{Game-Theoretic Nested Search (\algd)\xspace}
\newcommand{\algisne}{\textsc{isNashEquilibrium}\xspace}

\newcommand{\Xnear}{X_{\textup{near}}}
\newcommand{\Xgoal}{\X_{\textup{goal}}}

\newcommand{\xcand}{x_{\textup{cand}}}

\newcommand{\xnew}{x_{\textup{new}}}

\newif\ifincludeappendix
\includeappendixtrue  

\ifwafr
  \newcommand{\appref}[1]{%
    the extended version~\cite{engle2025effective}%
  }

  \newcommand{\appfigref}[2]{%
    the extended version~\cite{engle2025effective}%
  }
\else
  \newcommand{\appref}[1]{%
    App.~\ref{#1}%
  }

  \newcommand{\appfigref}[2]{%
    App.~\ref{#1}, Fig.~\ref{#2}%
  }
\fi

\SetKw{KwNot}{not}
\SetKw{KwOr}{or}
\SetKw{KwAnd}{and}
\SetKw{KwRet}{return} 

\begin{document}

    \title{Effective Game-Theoretic Motion Planning via Nested Search}
\author{Avishav Engle\ifwafr\thanks{Corresponding author.}\fi \and Andrey Zhitnikov \and Oren Salzman \and \\ Omer Ben-Porat \and  Kiril Solovey}
\institute{Technion--Israel Institute of Technology, Haifa, Israel\\
\email{\{avishav,andreyz\}@campus.technion.ac.il, osalzman@cs.technion.ac.il, \{omerbp,kirilsol\}@technion.ac.il}}
\authorrunning{Engle et al.}

\maketitle

\vspace{-5pt}
\begin{abstract}
To facilitate effective and safe deployment, individual robots must reason about interactions with other agents, which often occur without explicit communication. Recent work has identified game theory, particularly the concept of \emph{Nash Equilibrium} (NE), as a key enabler for behavior-aware motion planning. Yet, existing work falls short of fully unleashing the power of game-theoretic reasoning. Specifically, optimization-based methods require simplified robot dynamics and may get trapped in local minima due to convexification. Other works that rely on the explicit enumeration of all possible trajectories suffer from poor scalability. 
To bridge this gap, we introduce \emph{Game-Theoretic Nested Search} (GTNS), a scalable, and provably correct 
approach for computing NEs in 
general dynamical systems. 
GTNS efficiently searches the action space of all agents involved, while discarding trajectories that violate the NE constraint 
through an inner search over a lower-dimensional space.
Our algorithm enables explicit selection among equilibria by utilizing a user-specified global objective, thereby capturing a rich set of realistic interactions. We demonstrate the approach across a variety of autonomous driving and racing scenarios, achieving solutions in mere seconds on commodity hardware. 
\end{abstract}
\section{Introduction}
\label{sec:Intro}
\vspace{-5pt}
Modern robotic applications---from autonomous driving to assistive robotics---require behavior-aware motion planning that accounts for interactions between the ego robot and surrounding agents. These interactions often occur without explicit communication. Yet, to accomplish its mission effectively, the robot must reason about the behavior of other agents and the ways they respond to its actions. 
To address this challenge, this work seeks to develop game-theoretic motion planners that are both theoretically grounded and computationally efficient. 

The majority of approaches for behavior-aware motion planning consider a two-step process where the ego robot computes a best-response trajectory after predicting the actions of surrounding agents~\cite{rudenko2020human, lee2017desire,salzmann2020trajectron++,jiang2023motiondiffuser, schmerling2018multimodal}. 
Such approaches overlook that agents can update their actions based on other agents' behavior, leading to overly conservative solutions. 

More recent methods aim to model interactions between agents through a game-theoretic perspective, thereby assuming that agents are rational and optimize a private objective function (e.g., travel time, safety, lead over opponent), while accounting for the fact that other agents are involved in the same process.
For instance, some papers model interactions
via a Stackelberg equilibrium and develop approaches based on dynamic programming to approximate the
equilibrium \cite{fisac2019hierarchical}. However, Stackelberg games give an unfair advantage to one of the players \cite{le2022algames}, and hence lead to less realistic behavior. In the context of autonomous driving, such approaches can create overly aggressive maneuvers for one of the vehicles and be overly conservative for the rest.

More advanced methods model interactions via a Nash equilibrium (NE) \cite{nash1950equilibrium}, wherein no agent is better off by unilaterally deviating from the current solution. Although the game-theory community has extensively studied NE \cite{nisan2007agt,shoham2009mas}, most computationally-efficient results consider overly simplistic models that fail to capture 
the intricacies of real robotic systems, including continuous action and state spaces with non-convex constraints, and long-horizon strategies. Some papers have attempted to approximate an NE using relaxations and utilizing differential dynamic programming~\cite{fridovich2020efficient}, albeit with limited applicability to robotics due to the inability to capture collision-avoidance constraints. More specialized approaches that incorporate hard collision constraints and robot dynamics approximate NEs through optimization-based approaches \cite{wang2021game,le2022algames,spica2020real,le2021lucidgames}, albeit with limited theoretical justification. For example, \cite{le2022algames} utilizes a first-order approximation and, as such, cannot guarantee finding a global NE except in some special convex settings, which rarely occur in robotics. Specifically, the presence of obstacles, nonlinear dynamics, and multiple agents gives rise to multiple homotopy classes of NE solutions, which cannot be captured in a convex approximation. An iterative approach~\cite{wang2021game}, which relies on \cite{spica2020real}, falls short of the sufficient conditions for NE convergence, attaining only the necessary conditions. This is due to a lack of convexity in both cases, which follows from nonlinear dynamics and a non-convex state space. 

A different work~\cite{liniger2019noncooperative} considers a mathematically principled approach for efficiently obtaining a NE from a set of precomputed robot paths, under special conditions, such as right-of-way, which benefits one of the agents. Unfortunately, in a general setting, the approach boils down to a brute-force computation of the NE via the full payoff matrix, which is exponential in the number of robots. Moreover, the approach requires the explicit enumeration of all robot trajectories used, which can be costly if a rich set of strategies needs to be considered. 
A recent work~\cite{tikna2023graphs} 
invokes the bi-matrix game reasoning of \cite{liniger2019noncooperative} in a receding horizon fashion to choose the immediate next action in a Stackelberg game, where the robots are confined to motions on a precomputed quasi-kinodynamic \emph{graph} where edges represent Euler curves. Unfortunately, no equilibrium guarantees are given, and the use of Euler curves is highly problem- and model-specific.

An approach by LaValle and Hutchinson~\cite{lavalle2002optimal} also leverages graph search, explicitly constructing a frontier of all non-Pareto-dominated NE candidate trajectories for each node. While this yields equilibrium guarantees, it incurs a substantial computational burden due to the need to verify and maintain all such equilibria throughout the search. 

Learning-based approaches have been applied in behavior-aware settings in the context of head-to-head racing in computer games \cite{wurman2022outracing}, regulating mixed traffic flow \cite{wu2021flow}, and autonomous driving \cite{wang2025generativeaiautonomousdriving}, to name a few examples. However, as is typically the case with learning methods, they offer no explainability or generalizability and heavily rely on past experience for making correct decisions. Thus, their reliability remains unclear, especially in long tail events encompassing complex interactions between multiple agents \cite{wang2025generativeaiautonomousdriving}. 

\begin{figure}[t!]
  \centering
  \captionsetup[subfigure]{skip=2pt}
  \captionsetup{aboveskip=2pt, belowskip=0pt}

  \begin{subfigure}[b]{0.17\textwidth}
    \centering
    \begin{overpic}[%
        height=\linewidth,   
        trim=240pt 5pt 240pt 5pt, 
        clip,
        angle=90
    ]{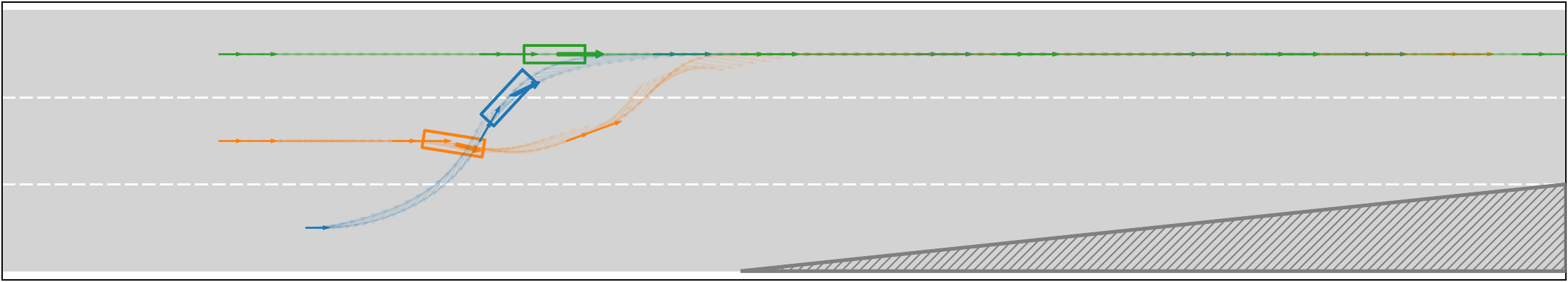}
      \put(10,91){%
        \fcolorbox{gray}{white}{\scriptsize
          \shortstack{(a) $\alpha = 0.00$\\$\lambda_{\mathrm{prox}} = 0.00$}%
        }%
      }
    \end{overpic}
    \phantomsubcaption\label{fig:highwaymerge000}
  \end{subfigure}\hspace{0.025\textwidth}
  \begin{subfigure}[b]{0.17\textwidth}
    \centering
    \begin{overpic}[%
        height=\linewidth,
        trim=240pt 5pt 240pt 5pt,
        clip,
        angle=90
    ]{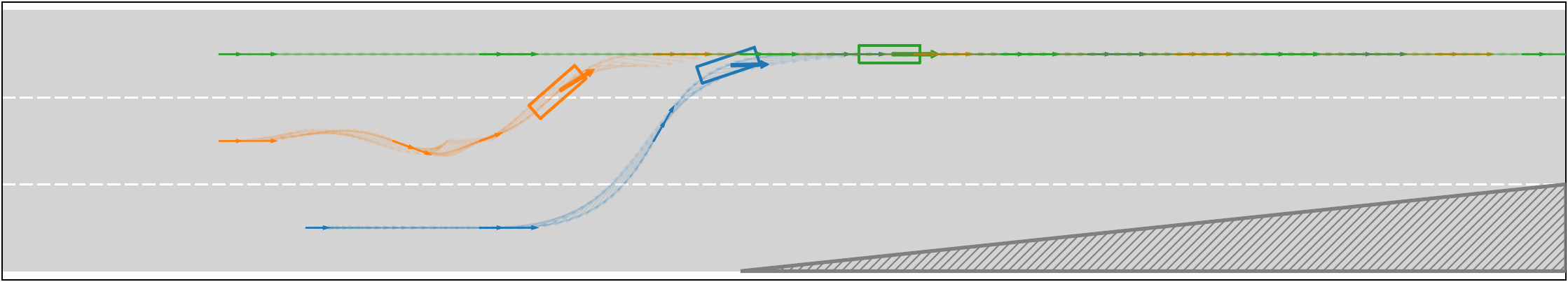}
      \put(10,91){%
        \fcolorbox{gray}{white}{\scriptsize
          \shortstack{(b) $\alpha = 0.25$\\$\lambda_{\mathrm{prox}} = 0.62$}%
        }%
      }
    \end{overpic}
    \phantomsubcaption\label{fig:highwaymerge025}
  \end{subfigure}\hspace{0.025\textwidth}
  \begin{subfigure}[b]{0.17\textwidth}
    \centering
    \begin{overpic}[%
        height=\linewidth,
        trim=240pt 5pt 240pt 5pt,
        clip,
        angle=90
    ]{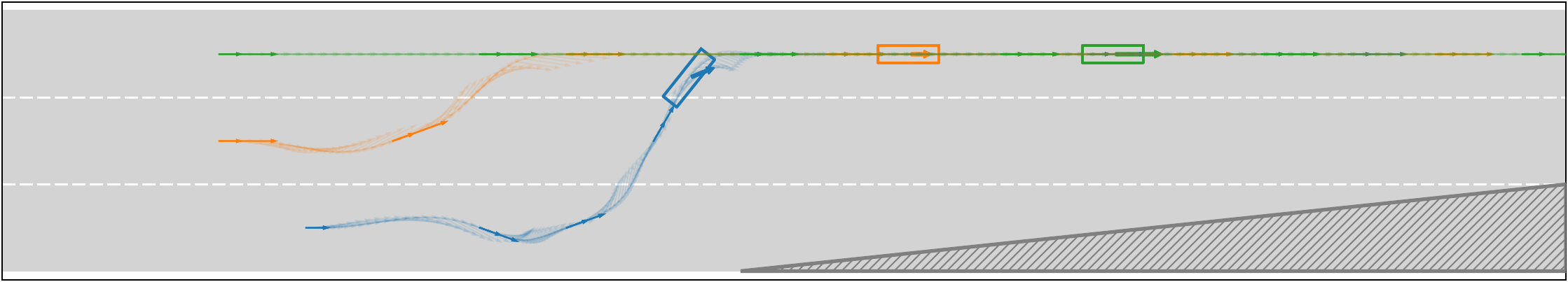}
      \put(10,91){%
        \fcolorbox{gray}{white}{\scriptsize
          \shortstack{(c) $\alpha = 0.75$\\$\lambda_{\mathrm{prox}} = 1.87$}%
        }%
      }
    \end{overpic}
    \phantomsubcaption\label{fig:highwaymerge075}
  \end{subfigure}\hspace{0.025\textwidth}
  \begin{subfigure}[b]{0.17\textwidth}
    \centering
    \begin{overpic}[%
        height=\linewidth,
        trim=240pt 5pt 240pt 5pt,
        clip,
        angle=90
    ]{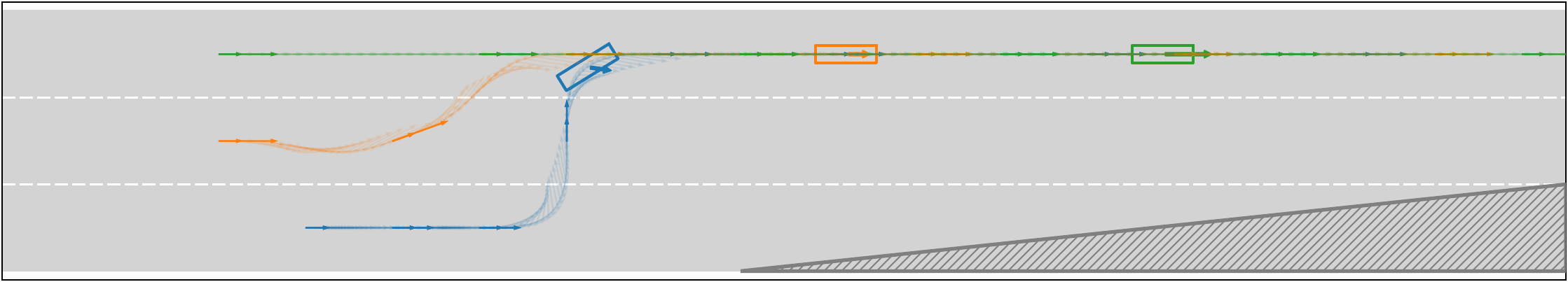}
      \put(10,91){%
        \fcolorbox{gray}{white}{\scriptsize
          \shortstack{(d) $\alpha = 1.00$\\$\lambda_{\mathrm{prox}} = 2.50$}%
        }%
      }
    \end{overpic}
    \phantomsubcaption\label{fig:highwaymerge100}
  \end{subfigure}
  \caption{Multi-lane merge: Robot 1 (blue) merges across Robot 2's (orange) lane into Robot 3's (green) lane. 
  Existing algorithms typically restricted to a single homotopy class, whereas our method---via tuning of the global objective (here, decreasing Robot 1's priority (larger $\alpha$ in $J{=}(1{-}\alpha)J^1{+}\alpha(J^2{+}J^3)$) and increasing proximity penalty weight $\lambda_{\mathrm{prox}}$)---yields distinct NE behaviors: (a) aggressive zip-merge; (b) zip-merge; (c) yield-then-merge; (d) over-cautious.}

  \label{fig:highwaymerge}
\end{figure}

\niceparagraph{Contribution.} We introduce an efficient and provably correct method, termed \algdfull, for computing NEs in motion planning.
Our method guarantees convergence to an NE without restricting the problem domain (e.g., convexity) and applies to general objective functions and dynamical constraints of individual robots. 
Furthermore, the approach enables the explicit tuning of the NE solution retrieved via a user-specified global objective function, ensuring that the returned NE minimizes this function among all available NEs (see, e.g., Fig.~\ref{fig:highwaymerge}).

At the core of our methodology is a nested-search approach consisting of an outer and inner search method. The outer search implicitly explores a tree of potential multi-robot trajectories, while the inner search is invoked to verify that a given solution satisfies the NE requirement by considering improving unilateral deviations of individual robots. For computational efficiency, our nested-search approach implicitly explores the joint state space of the multi-robot system via graphs that efficiently encode the trajectory space of individual robots while adhering to the robots' full dynamical constraints.

This search-based approach allows for an implicit encoding of trajectory-based constraints and incremental reuse of information, which is otherwise difficult to realize efficiently within an optimization-based approach \cite{9922600} or in methods that require explicit enumeration of all possible trajectories~\cite{liniger2019noncooperative,tikna2023graphs}. Our algorithm has conceptual similarities to the LaValle-Hutchinson approach~\cite {lavalle2002optimal}, but crucially does not explicitly maintain all candidate NE trajectories, and employs a powerful pruning mechanism to discard non-NE trajectories (and their subtrees), rendering the search over complex and high-dimensional spaces tractable.

We demonstrate the approach on a variety of realistic autonomous driving and racing scenarios (including intersections, lane merges, following distance, opposing-lane passes, and racetrack overtakes) where we achieve solutions in mere seconds on commodity hardware.
From the theory side, our algorithmic framework leads to a rigorous but natural proof for the attainment of an NE.



Finally, we mention that our framework directly supports scalable generation of training data for generative AI and foundation models in autonomy~\cite{lee2017desire,salzmann2020trajectron++,jiang2023motiondiffuser,gao2025foundation}. It enables rich simulation of multi-agent dynamics across domains (traffic, racing, aerial swarms, aircraft sequencing, warehouses) and can efficiently synthesize rare, safety-critical cases for training and evaluation. 

\niceparagraph{Organization.} Sec.~\ref{sec:Preliminaries} provides preliminaries and our problem statement. Sec.~\ref{sec:graph} considers a graph-based formulation of the above problem. Sec.~\ref{sec:Algorithm} describes our \algd algorithm, whose theoretical guarantees are  discussed in Sec.~\ref{sec:theory}. Sec.~\ref{sec:Experiments} details  experimental results. Conclusion and future work appears in Sec.~\ref{sec:conclusion}.


\section{Preliminaries} 
\label{sec:Preliminaries}

We begin by introducing the notation for kinodynamic multi-robot motion planning, then formalize the underlying game-theoretic planning problem. Consider $m \geq 2$ robots indexed by $i \in [m] := \{1, \dots, m\}$.  
Each robot $i$ has a state space $\mathcal{X}^i \subseteq \mathbb{R}^{d_i}$ and a control space $\mathcal{U}^i \subseteq \mathbb{R}^{D_i}$.  
Its motion is governed by the 
dynamics
\begin{equation}
    \dot{x}^i(t) = f^i\bigl(x^i(t), u^i(t)\bigr),
    \quad x^i(t) \in \mathcal{X}^i,\; u^i(t) \in \mathcal{U}^i.
    \label{eq:MotionIndividualRobot}
\end{equation}
Additionally, the robot must remain within the free space $\mathcal{X}_f^i \subseteq \mathcal{X}^i$ to avoid collisions with static obstacles (collisions with other robots discussed below).  

For a given planning horizon $T > 0$, let $\mathcal{U}^i_T$ be the set of admissible control functions $u^i:[0,T]\mapsto \mathcal{U}^i$ for robot $i$.  
Given an initial state $x_0^i \in \mathcal{X}^i$ and a control function $u^i \in \mathcal{U}^i_T$, the induced trajectory is the mapping
$
\pi^i_{x_0^i,u^i} : [0,T] \mapsto \mathcal{X}^i,
$
where $\pi^i_{x_0^i,u^i}(\cdot)$ is a solution to Eq.~\eqref{eq:MotionIndividualRobot} with $\pi^i_{x_0^i,u^i}(0)=x_0^i$.  The set of feasible trajectories for robot $i$ is
$
\Pi^i_T := \{\, \pi^i_{x_0^i,u^i} \mid x_0^i \in \mathcal{X}^i,\; u^i \in \mathcal{U}^i_T \,\}.
$

Next, we generalize the above definitions to encompass the full multi-robot system. The joint state space of the $m$ robots is denoted by $\X := \times_{i=1}^m \X^i$, i.e.,  a point in $\X$ describes the simultaneous state of all the robots in the system. Similarly, the joint control space is $\mathcal{U} := \times_{i=1}^m \mathcal{U}^i$ and the joint control space within specific time horizon $T$ is  $\mathcal{U}_T := \times_{i=1}^m \mathcal{U}^i_T$. 
Given initial joint state $x_0=(x_0^1,\ldots,x_0^m) \in \mathcal{X}$ and joint control $u=(u^1,\dots,u^m)\in \mathcal{U}_T$, the induced joint trajectory $\pi_{x_0,u}=(\pi_{x_0^1,u^1}^1,\ldots, \pi_{x_0^m,u^m}^m)$ specifies the trajectories of each of the $m$ robots, and can be interepreted as the mapping
$
\pi_{x_0,u} : [0,T] \mapsto \mathcal{X}.
$
The set of all joint trajectories is
$
\Pi_T := \{\, \pi_{x_0,u} \mid x_0 \in \mathcal{X},\; u \in \mathcal{U}_T \,\}.
$

Each robot $i$ evaluates its performance via a cost functional $
J^i : \Pi^i_T \times \Pi^{-i}_T \mapsto \mathbb{R}_{\ge 0},$
where $\Pi^{-i} := \times_{j\neq i}\Pi^j$ is the set of trajectories of the remaining robots with indices $[m]\setminus \{i\}$.  
For trajectories $\pi^i \in \Pi^i_T$ and $\pi^{-i}\in \Pi^{-i}_T$, we define
\begin{equation}
\label{eq:IndividualCost}
    {\small J^i(\pi^i, \pi^{-i})
    := \int_0^T c^i\bigl(\pi^i(t), \pi^{-i}(t)\bigr)\, \mathrm{d}t,}
\end{equation}
with stage cost $c^i$ depending on states and controls.  

\subsection{Problem Definition}
We seek to compute joint trajectories that satisfy the following requirement. 

\begin{definition}[Nash Equilibrium]
A joint trajectory $\pi = (\pi^1,\dots,\pi^m) \in \Pi_T$ is a (pure) Nash equilibrium (NE) if, for every robot $i \in [m]$ and for all alternative feasible trajectories $\tilde \pi^i \in \Pi^i_T$ (from the same start state), no robot can reduce its cost by 
deviating while the other robots fix their trajectories, i.e.,
    $J^i(\pi^i, \pi^{-i}) \;\leq\; J^i(\tilde \pi^i, \pi^{-i})$.
\end{definition}

Moreover, we wish to compute an NE that minimizes a given global cost.  

\begin{problem}[Optimal NE Planning]\label{problem:NE}
Fix a planning horizon $T>0$ and consider a global cost functional 
\begin{equation}
\label{eq:GlobalCost}
   {\small J(\pi) := \int_0^T c\bigl(\pi^1(t),\ldots, \pi^m(t)\bigr)\, \mathrm{d}t ,
   \quad \pi \in \Pi_T,}
\end{equation}
with stage cost $c$ depending on joint states and controls. We wish to find a joint trajectory $\pi^* \in \Pi_T$ (and its corresponding control laws) such that: (1) $\pi^*$ is a Nash equilibrium (Def.~\ref{problem:NE}) that minimizes the global cost $J(\pi^*)$ among all NEs; (2) the final state of $\pi^*$ is in the joint goal region, i.e.,  $\pi^*(T)\in \mathcal{X}_{\mathrm{goal}} := \times_{i=1}^m \mathcal{X}^i_{\mathrm{goal}}$; (3) $\pi(t)\in \mathcal{X}_f$ for all $t \in [0,T]$.
\end{problem}

\niceparagraph{Discussion.}  We consider an \emph{open-loop} setting: an NE solution is computed once for a horizon and then executed, rather than replanned at every time step. This setting is common in robotics (see, e.g.,~\cite{ZhuB23}), motivated by the high computational demands of computing NEs, which currently preclude real-time feedback control. The resulting NE solutions promote stable execution, since local controllers cannot improve their cost by deviating from the prescribed trajectory.  


We seek \emph{global} NEs, where robots may choose any action in their control spaces 
This contrasts with \emph{local} NEs, which are implicitly considered in, e.g.,~\cite{le2022algames,wang2021game,spica2020real,rowold2024open,fridovich2020efficient}, where the solution is restricted to a subset of the joint state space (e.g. due to linearization and convexification). Local NEs may not correspond to any global NE, leading to unstable or otherwise poorly chosen execution. See illustration in \appfigref{app:figures}{fig:localvsglobal}. Another key difference from prior work is that we require the chosen NE to optimize a global objective. This promotes predictability and explainability, as (under mild conditions) a unique global optimum exists, ensuring all robots converge to the same NE. 


A remaining challenge is incentivizing robots to follow the chosen solution. An NE optimizing a specific cost function $J$ may favor certain robots, motivating others to defect. In cooperative domains (e.g., driving), $J$ can encode social welfare or fairness and may even be set by an external arbitrator (e.g., to resolve right-of-way). In competitive settings, the more general solution concept of correlated equilibrium (CE)~\cite{aumann1974subjectivity} could be applied: a mediator samples from a distribution and privately recommends actions, which can enable coordination on outcomes that are more efficient or fair. We leave this for future work.

\begin{figure}[t]
  \centering
  \includegraphics[width=0.9\textwidth]{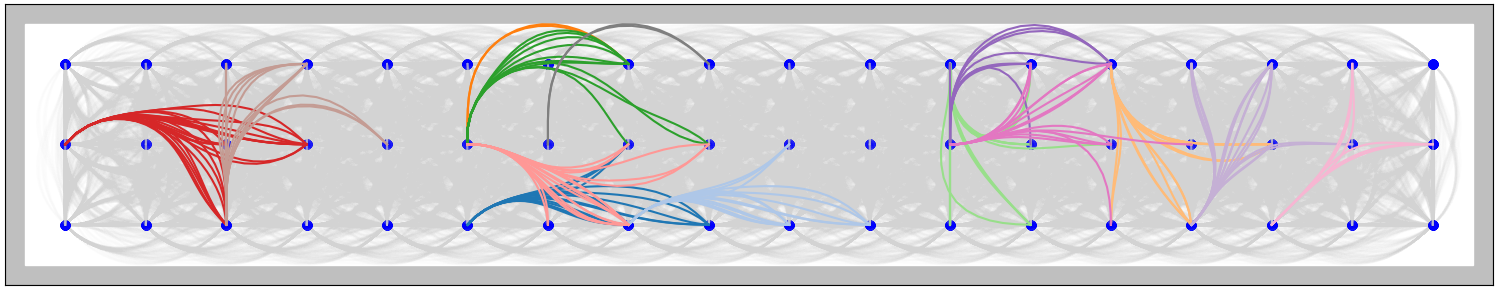}
  \caption{\small Kinodynamic grid graph. At each $(x,y)$ lattice coordinate (blue dots) there are various $\theta,\ v$ and $\delta$ values. Graph edges are gray; a randomly selected subset is emphasized and colored by the source state.}
  \label{fig:gridgraph}
\end{figure}

\section{Graph-Based setting \label{sec:graph}}
To tackle Problem~\ref{problem:NE}, we draw inspiration from previous work on multi-robot motion planning, wherein the curse of dimensionality, inherent in exploring the joint state space, is alleviated by first capturing the state space of the individual robots via discrete graphs, and then planning their collective motion on an implicitly-defined tensor graph~\cite{solovey2016finding,ShomeSDHB20,lavalle2002optimal}. In our setting, however, we utilize specialized constructions of those graphs to accommodate differential constraints \eqref{eq:MotionIndividualRobot} and, most importantly, game-theoretic aspects. In preparation for our algorithm, which is described in the next section,  we describe here the graph-based setting, as well as a graph-based counterpart of Problem~\ref{problem:NE}.

For each robot $i \in [m]$, we construct a \emph{directed} graph 
$
G^i := (V^i,E^i),
$
that embodies its dynamically feasible motions.  
The vertex set $V^i \subseteq \mathcal{X}^i$ consists of sampled states, including the start state $x_0^i$ and states in the goal region $\mathcal{X}^i_{\mathrm{goal}}$.  
Each edge $e^i=(x^i,y^i) \in E^i$ corresponds to a dynamically feasible local trajectory $\pi^i_{e^i}$ from the state $x^i$ to the state $y^i$ obtained under some control law $u^i_{e^i} \in \mathcal{U}^i_{\Delta t}$, executed over a fixed duration $\Delta t > 0$ from the state $x^i\in \X^i$.

Every edge, therefore, represents a time-parameterized trajectory of duration $\Delta t$, making $G^i$ a kinodynamic variant of the probabilistic roadmap (PRM)~\cite{kavraki1996probabilistic}, as used in, e.g.,~\cite{SchmerlingJP15,schmerling2015optimal}. We note that each edge in the graph is a curve (kinodynamic-adhering trajectory) induced by the corresponding control function and not a straight line. 
We consider an explicit representation of each $G^i$; more efficient semi-implicit constructions~\cite{ShomeK21,HonigHT22} are left for future work.

From the individual graphs $\{G^i\}_{i\in[m]}$, we form the \emph{tensor-product directed graph} $G=(V,E)$ with
$V = \times_{i=1}^m V^i$, and  $E = \times_{i=1}^m E^i$.     
A joint vertex 
    $x=(x^1,\dots,x^m)\in V$ 
encodes a joint state of the $m$ robots, while a joint edge $e=(e^1,\dots,e^m)\in E$ represents a local joint trajectory $\pi_e=(\pi^1_{e^1},\ldots,\pi^m_{e^m})$ synchronized motion where each robot $i$ follows the trajectory $\pi^i_{e^i}$ associated with the edge $e^i\in E^i$ during the same duration $\Delta t$.  
Since $|V|$ and $|E|$ grow exponentially in $m$, we never construct $G$ explicitly; instead, we explore it implicitly by combining neighbors from the explicit graphs $\{G^i\}$, where an A*-search directs which portion of $G$ to explore next (see below).

Next, we adapt Problem~\ref{problem:NE} to the graph-based setting. Fix a number of steps $n \in \mathbb{N}_{+}$, such that $n\Delta t$ is equal to the planning horizon $T > 0$.  
Let a trajectory $\pi^i_n$ of length $n$ consist of a sequence of $n$ edges on the graph $G_i$.
That is, given the graph path $e_1^i,\ldots,e_n^i$, the trajectory $\pi^i_n$ is the concatenation of the trajectories $\pi^i_{e_1^i},\ldots,\pi^i_{e_n^i}$. 
We denote by $\Pi^i_n$ the set of $n$-step trajectories of robot $i$. 
We similarly define the joint trajectory set $\Pi_n:=\times_{i=1}^m \Pi^i_n$. Notice that each joint trajectory $\pi_n\in \Pi_n$ corresponds to a path on the graph $G$. 

\begin{definition}[Graph Nash Equilibrium]
\label{def:DiscreteNash}
For a horizon of $n$ steps, a joint discrete trajectory $\pi_n \in \Pi_n$ is a \emph{graph Nash equilibrium (gNE)} if, for every robot $i \in [m]$ and every alternative trajectory $\tilde \pi^i_n \in \Pi^i_n$ (from the same start state), no robot can reduce its cost by unilaterally deviating to another feasible $n$-step trajectory, i.e.,
    $J^i(\pi^i_n, \pi^{-i}_n) \;\leq\; J^i(\tilde \pi^i_n, \pi^{-i}_n)$.
\end{definition}

\begin{problem}[Optimal gNE Planning]\label{problem:DiscreteNE}
Fix $n > 0$. Given a global cost 
$J:\Pi_n \mapsto \mathbb{R}_{\geq 0}$, 
the objective is to find $\pi^*_n \in \Pi_n$ such that: (1) $\pi^*_n$ is a gNE (Def.~\ref{def:DiscreteNash}) that minimizes the global cost $J(\pi^*_n)$ among all gNEs; (2) $\pi^*_n$ reaches the joint goal region $\mathcal{X}_{\mathrm{goal}}$ in $n$ steps; (3) $\pi^*_n$ is collision--free, i.e., all intermediate joint states, and the continuous motions along edges, respect $\mathcal{X}_f$.
\end{problem}

\niceparagraph{Discussion.} We discuss the interrelation between Problem \ref{problem:NE} and Problem \ref{problem:DiscreteNE}. Previous work has shown that for the geometric setting (i.e., without kinodynamic constraints), the graphs $G^1,\ldots, G^m$ can be constructed such that any feasible multi-robot trajectory in the continuous state space of the robot group can be approximated arbitrarily well in the graph $G$~\cite{ShomeSDHB20,Dayan.ea.23}. 
The extension to the kinodynamic setting can be done via recent analysis techniques for asymptotically optimal sampling-based planning~\cite{kleinbort2020refined}. We leave this technical proof for future work. We note that the points highlighted in the discussion in Sec.~\ref{sec:Preliminaries}, are equally applicable to the graph setting.

\section{Algorithm for graph NE} \label{sec:Algorithm}
We describe our algorithmic approach, \algdfull, to tackle Problem~\ref{problem:DiscreteNE}. 
The algorithm employs a nested approach, consisting of an outer search (Alg.~\ref{alg:DSP}) and an inner search (Alg.~\ref{alg:isEquilibriumStar}).
Both levels of search are performed by A*-based search algorithms \cite{hart1968formal}, albeit on different graphs and for different purposes. We describe the  high-level structure of the algorithm  (see additional implementation optimizations in \appref{app:speedup}).

\niceparagraph{Outer search.}
The \algd algorithm (Alg.~\ref{alg:DSP}) explores the implicit tensor-product graph $G$ induced by the per-robot graphs $\{G^i\}_{i\in[m]}$ each produced by \texttt{generateRobotRoadmap}. These individual graphs can be precomputed once (offline) and reused across scenarios for the same robot platform. By construction, motion along edges of $G$ is dynamically feasible for all robots; we additionally enforce collision avoidance and the gNE condition.

The A* search maintains a priority queue \texttt{OPEN} keyed by $J(\pi_x)+H(x)$ (where $H(\cdot)$ is a cost-to-go heuristic)) and a predecessor map $P$ for path reconstruction. At each iteration, the algorithm:
(i) extracts the minimum-key node $x$ (line~\ref{alg:SelectNodeLine}); 
(ii) recovers the current best path $\pi_x$ from the root using $P$ (line~\ref{alg:TrajectoryFromRootLine}); and
(iii) enumerates neighbors $\Xnear$ of $x$ in $G$ (line~\ref{alg:PullNeighborsLine}). 
Here, a joint node $y=(y^1,\dots,y^m) \in V$ belongs to $\Xnear$ iff for every $i$ there is an edge $(x^i, y^i)\in E^i$. 

For each potential neighbor $\xnew\in\Xnear$, the algorithm (a) verifies that the local edge trajectory $\pi_{x\to \xnew}$ is collision-free (line~\ref{alg:CollisionCheckLine}); (b) forms the candidate path 
$\pi_{\xnew} := \texttt{concatenate}(\pi_x,\pi_{x\to\xnew})$ (line~\ref{alg:ConcatLine}); and 
(c) checks improvement w.r.t.\ any incumbent path to $\xnew$ (line~\ref{alg:BetterInClosedLine}), where the incumbent is recovered via $P$ (line~\ref{alg:OldTrajLine}) and updated (line~\ref{alg:UpdateParentLine}). 
It then calls the predicate \algisne\ on $\pi_{\xnew}$ (line~\ref{alg:LowLevelSearchCall}); only if the candidate path satisfies the gNE condition and offers an improved cost over the incumbent path (similar to standard A* dominance checks), is $\xnew$ inserted into \texttt{OPEN} with key $J(\pi_{\xnew})+H(\xnew)$. If a goal node is popped and validated, the corresponding trajectory is returned.
{\small
\begin{algorithm}[h!]
  \caption{\algd}
  \label{alg:DSP}

  \texttt{OPEN} $\gets \{x_0\}$; \quad $P(x)\gets \emptyset,\ \forall x\in V$ \\
  $\left\{G^i \gets \texttt{generateRobotRoadmap}(f^i{,}J^i{,}x_0^i{,}\X_{\text{goal}}^i)\right\}_{i\in [m]}$ \\ [-0.25em]

  \While{\KwNot \textup{\texttt{OPEN.empty()}}}{
    $x \gets$ \texttt{OPEN.extractMin()} \\ \label{alg:SelectNodeLine}
    $\pi_x \gets \texttt{trajFromRoot}(P, x)$ \\ \label{alg:TrajectoryFromRootLine}
    $\Xnear \gets G.\texttt{neighbors}(x)$ \\ \label{alg:PullNeighborsLine}

    \For{\textup{$\xnew \in \Xnear$}}{ \label{alg:ExtendNodeLine}
      \If{\textup{\texttt{collisionFreeEdge}$(\pi_{x\to \xnew})$}}{
        \label{alg:CollisionCheckLine}
        $\pi_{\xnew} \gets \texttt{concatenate}(\pi_x, \pi_{x\to \xnew})$ \\ \label{alg:ConcatLine} 
        $\pi_{\xnew}^{\textup{old}} \gets \texttt{trajFromRoot}(P, \xnew)$ \Comment{incumbent traj. to $\xnew$} \\ \label{alg:OldTrajLine}

        \If{\textup{\texttt{notYetVisited}$(\xnew)$ \KwOr $J(\pi_{\xnew}) < J(\pi_{\xnew}^{\textup{old}})$}}{
          \label{alg:BetterInClosedLine}
          \If{\algisne$(\pi_{\xnew})$}{
            \label{alg:LowLevelSearchCall}
            $P(\xnew)\gets x$  \Comment{store parent} \\ \label{alg:UpdateParentLine}
            $\texttt{OPEN.insert}(\xnew, J(\pi_{\xnew}) + H(\xnew))$ \\

            \If{$\xnew \in \Xgoal$}{
              \Return{$\pi_{\xnew}$} \\
            }
          }
        }
      }
    }
  }
  \Return{$\emptyset$} \\
\end{algorithm}
}

\niceparagraph{Inner search.}
\algisne\ is a best-response oracle: it solves a single-robot motion planning problem on $G^i$ to see whether any agent has a strictly cheaper unilateral deviation of the same horizon that is dynamically feasible and collision-free w.r.t.\ others’ fixed paths. For each candidate joint trajectory $\pi_{\xcand}=(\pi^i_{\xcand^i},\pi^{-i}_{\xcand^{-i}})$ 
 encountered by \algd, \algisne\ (Alg.~\ref{alg:isEquilibriumStar}) runs, for each $i\in[m]$, an A* search on $G^i$ subject to three simple constraints: (1) The search goal is set (line~\ref{alg:initspecificgoal}) to the \emph{specific} terminal configuration $\xcand^i$ appearing at the end of $\pi^i_{\xcand^i}$ (no goal \emph{region}); (2) Only paths of \emph{exactly} $n$ edges (garnered by the method $\texttt{numEdges}$) are considered competitive (the same length as $\pi_{\xcand}$). Nodes are only considered until the partial path length $d$ reaches $n$; (3) Edges for robot $i$ are collision-checked ($\texttt{collisionFreeEdge}$, line~\ref{alg:DynamicObstaclesCollisionCheckLine}) against the \emph{fixed, time-parameterized} trajectories of the other robots $\pi^{-i}_{\xcand^{-i}}$.

Operationally, for each $i$ the routine initializes \texttt{OPEN} and a predecessor map $P^i$, then performs the usual A$^\ast$ node selection, neighbor expansion, concatenation, and incumbent comparison on the single-robot graph $G^i$ (with key $J^i(\cdot)+H^i(\cdot)$), exactly mirroring the outer loop but \emph{only} for agent $i$ and under the constraints above. If it ever reaches $\xcand^i$ in exactly $n$ steps with a trajectory $\tilde\pi^i_{\xcand^i}$ yielding 
$J^i(\tilde\pi^i_{\xcand^i}) < J^i(\pi^i_{\xcand^i})$, the routine immediately returns \textsc{False} (the candidate joint trajectory does not meet the gNE conditions). If no robot can improve, the routine returns \textsc{True}. 
{\small
\begin{algorithm}[h!]
\caption{\algisne$(\pi_{\xcand})$}
\label{alg:isEquilibriumStar}
  $n \gets \texttt{numEdges}(\pi_{\xcand})$ \\

  \For{$i \in [m]$}{
    $\xcand^i \gets$ last vertex in $\pi^i_{\xcand^i}$ \Comment{$\pi_{\xcand} = (\pi^i_{\xcand^i}, \pi^{-i}_{\xcand^{-i}})$} \\ \label{alg:initspecificgoal}
    \texttt{OPEN} $\gets \{x^i_0\}$; \quad $P^i(x^i) \gets \emptyset,\ \forall x^i \in V^i$ \\
    \While{\textup{\texttt{not OPEN.empty()}}}{
      $x^i \gets$ \texttt{OPEN.extractMin()} \\
      $\tilde{\pi}^i_{x^i} \gets \texttt{trajFromRoot}(P^i, x^i)$ \\
      $d \gets \texttt{numEdges}(\tilde{\pi}^i_{x^i})$ \\

      \If{$d \geq n$}{
        \textbf{continue} \\
      }

      $\Xnear^i \gets G^i.\texttt{neighbors}(x^i)$ \\

      \For{$\xnew^i \in \Xnear^i$}{
        \If{\textup{$\texttt{collisionFreeEdge}(\tilde{\pi}^i_{x^i \to \xnew^i}, \pi^{-i}_{\xcand^{-i}})$}}{
         \label{alg:DynamicObstaclesCollisionCheckLine}
          $\tilde{\pi}^i_{\xnew^i} \gets \texttt{concatenate}(\tilde{\pi}^i_{x^i}, \tilde{\pi}^i_{x^i \to \xnew^i})$ \\
          $\tilde{\pi}_{\xnew^i}^{i,\ \textup{old}} \gets \texttt{trajFromRoot}(P^i, \xnew^i)$ \\

          \If{\textup{\texttt{notYetVisited}$(\xnew^i)$ \KwOr $J^i(\tilde{\pi}^i_{\xnew^i}) < J^i(\tilde{\pi}_{\xnew^i}^{i,\ \textup{old}})$}}{
            $P^i(\xnew^i) \gets x^i$ \\
            $\texttt{OPEN.insert}(\xnew^i, J^i(\tilde{\pi}^i_{\xnew^i}) + H^i(\xnew^i))$ \\

            \If{$\xnew^i = \xcand^i$ \KwAnd $d+1 = n$}{
              \If{$J^i(\tilde{\pi}^i_{\xcand^i}) < J^i(\pi^i_{\xcand^i})$}{
                \Return{FALSE} \\
              }
            }
          }
        }
      }
    }
  }
  \Return{TRUE}
\end{algorithm}
}


\section{Theoretical guarantees}\label{sec:theory}
We prove that Alg.~\ref{alg:DSP} provides an optimal solution to Problem~\ref{problem:DiscreteNE}. 
In preparation, we obtain the following lemma that states that the NE property is monotone, i.e., every subtrajectory of a NE trajectory is also NE. 
Given a trajectory $\pi_n$ of $n$ edges, the notation $\pi_{n'}$ for some $0<n'<n$ denotes the subtrajectory of $\pi_n$ consisting of the first $n'$ edges. Namely, $\pi_{n'}$ has a duration of $n'\cdot \Delta t$ and $\pi_{n'}(\tau)=\pi_{n}(\tau)$ for any $0\leq \tau\leq \Delta t\cdot n'$.


\begin{lemma}[Monotonicity of Nash Equilibrium]
    \label{lem:NashMonotonicity}
     If a trajectory $\pi_n$ is a gNE (Def.~\ref{def:DiscreteNash}) then any subtrajectory $\pi_{n'}$, where $0\leq n'< n$, is a gNE too. 
 \end{lemma}

\begin{proof}
First, recall that we consider  additive cost functions $J,\{J^i\}$. 
 Suppose that $\pi_n$ is a  gNE, namely for any robot $i\in [m]$, and any alternative trajectory $\tilde{\pi}_n^i\in \Pi^i_n$, 
 with the same start and final $\X^i$ state, i.e., $\pi^i_n(0)=\tilde{\pi}^i_n(0)$, and $\pi^i_n(n\cdot \Delta t)=\tilde{\pi}^i_n(n\cdot \Delta t)$, respectively, it holds that 
	$J^{i}(\pi_n^i,\pi_n^{-i}) \leq  J^{i}(\tilde{\pi}_n^i,\pi_n^{-i})$. 

Assume, by contradiction, that the subtrajectory $\pi_{n'}$, for some $0<n'<n$, is not a gNE: for some robot $i \in [m]$, there exists $\tilde{\pi}^i_{n'}\in \Pi^i_{n'}$ satisfying the collision constraints, $\pi^i_{n'}(0)=\tilde{\pi}^i_{n'}(0)$, and $\pi^i_{n'}(n'\cdot \Delta t)=\tilde{\pi}^i_{n'}(n'\cdot \Delta t)$ and inequality 
      $J^i(\pi^i_{n'}, \pi^{-i}_{n'}) > J^i(\tilde{\pi}^i_{n'}, \pi^{-i}_{n'})$.
Consider the remaining portion of  $\pi_{n}$ from the $n'+1$ edge, denoted by  $\pi_{n'+1:n}$. I.e., $\pi_{n'+1:n}$ has a duration of $(n-n')\cdot \Delta t$ and $\pi_{n'+1:n}(\tau)=\pi_{n}(\tau + n'\cdot \Delta t)$. Next, denote by $\tilde{\pi}^i_n$ the concatenation of  $\tilde{\pi}^i_{n'}$ and $\pi_{n'+1:n}$. This new trajectory is valid as $\tilde{\pi}^i_{n'}$ ends where $\pi_{n'+1:n}$ begins, and both those subtrajectories are collision free. By exploiting the additivity of the cost function $J^i$, we upper bound the $J^i$ cost resulting from  $\tilde{\pi}^i_{n'}$: 
{\small\begin{align*}
     J^i(\tilde{\pi}^i_{n}, \pi^{-i}_{n})& = J^i(\tilde{\pi}^i_{n'}, \pi^{-i}_{n'}) + J^i(\pi^i_{n'+1:n}, \pi^{-i}_{n'+1:n}) \\ & 
    < J^i(\pi^i_{n'}, \pi^{-i}_{n'}) + J^i(\pi^i_{n'+1:n}, \pi^{-i}_{n'+1:n})  = J^i(\pi^i_{n}, \pi^{-i}_{n}). 
\end{align*} }
This contradicts $\pi_n$ being a gNE. Thus, the gNE property is monotone. \qed 
\end{proof}

Combining the above lemma, with the optimality guarantee of A*, which is the backbone of both Alg.~\ref{alg:DSP} and~\ref{alg:isEquilibriumStar}, we establish the correctness of \algd.

\begin{theorem}
     \algd finds the optimal solution to Problem~\ref{problem:DiscreteNE}.  
\end{theorem}

\begin{proof}
  The A* algorithm returns an optimal solution when running on a search problem $P$ using an admissible heuristic \cite{hart1968formal}. 
    Recall that a search problem is a tuple 
    $P = \langle S, \rm{Succ}, s_{\rm begin}, S_{\rm goal} \rangle$,
    where 
    $S$ is a set of (abstract) states, 
    $\rm{Succ}: S \rightarrow 2^S$ is a succesor function mapping each state to a set of neighboring states, 
    $s_{\rm begin} \in S$ is the initial state
    and
    $S_{\rm goal} \subset S$ a set of goal states.

    We start by formalizing our problem as a search problem and then continue to show that Alg.~\ref{alg:DSP}, which can be seen as an adaptation of A*, inherits its optimality.
    Specifically, consider the search problem defined over the set of states $S=V$ (i.e., each abstract state represents a joint vertex, which in itself represents a joint robot state).
    The successor function is defined such that for vertices $v,v'\in V$, we have that
    $v' \in \rm{Succ}(v)$ if and only if
    $(v,v') \in G$.
    %
    Finally, the start and goal states are $x_0$ and states satisfying  Problem 2.

    Alg.~\ref{alg:DSP} can be seen as an adaptation of A* where
    nodes are only added to the \texttt{OPEN} list if they satisfy equilibrium conditions (line~\ref{alg:LowLevelSearchCall}) and using an admissible  heuristic $H$. 
    Pruning search nodes  in this search problem does not hinder optimality of the search due to the  Nash monotonicity property (Lemma~\ref{lem:NashMonotonicity}). That is, search nodes that are pruned cannot contribute to a longer trajectory that is a NE. Thus, optimality follows from the optimality of A*.
    As noted, Alg.~\ref{alg:isEquilibriumStar} is too an adaptation of A*, and thus when used with an admissable heuristic, returns an optimal, collision-free path of upto length $n$.
    \qed
\end{proof}

\niceparagraph{Time complexity.} The branching factor of node expansion in \algd grows exponentially with the number of robots. Even so, our approach efficiently copes with realistic problem settings in a few seconds (see Sec.~\ref{sec:Experiments}), due to the aggressive pruning of \algisne, which substantially reduces the search space, as compared to the non-game-theoretic and centralized multi-robot motion planning setting~\cite{WagnerC15}. We discuss additional speed-up techniques in Sec.~\ref{sec:conclusion}.

\section{Experiments and Results}\label{sec:Experiments}
We provide an experimental evaluation of our game-theoretic approach in driving and racing scenarios. We begin with implementation details and experiment setup, and then proceed to a qualitative evaluation showcasing the behaviors of the approach. 
Next, we evaluate the computational aspects of our approach and conclude with a comparison with LaValle-Hutchinson~\cite{lavalle2002optimal}. An additional experiment, which studies the effect of NE mismatch, is found in \appref{app:mismatch}.

\subsection{Implementation details and problem setting}
The single-robot graphs $G_i$ are built offline in Python using CasADi \texttt{Opti} \cite{Andersson2019} for BVP solving (see more information below). \algd is implemented in C++ using the LEMON \cite{dezsHo2011lemon} graph library. Simulations were run on a 64-bit Win11 laptop with an Intel Core i9-14900HX CPU (2.20 GHz) and 32 GB RAM.

For the robot model, we consider the second-order bicycle model \cite{lavalle2006planning}. 
The robot has the states, controls, and dynamics, of the form 
\[
{\small \mathbf{x}=\begin{bmatrix}x \\  y\\\theta\\v\\ \delta \end{bmatrix},\quad
\mathbf{u}=\begin{bmatrix} a\\ \omega \end{bmatrix},\quad
\dot{\mathbf{x}}=
\begin{bmatrix}
v\cos\theta\\
v\sin\theta\\
\tfrac{v}{L}\tan\delta\\
a\\
\omega
\end{bmatrix},}
\]
respectively. Here, $x$, $y$ and $\theta$ are the 2-D positions and heading of the car, $v$, and $\delta$ are its speed and steering angle, and $a$ and $\omega$ control their rate of change. $L$ is the distance between front and rear axles (see visualization in \appfigref{app:figures}{fig:secondordermodel}).  

\niceparagraph{Kinodynamic graphs.}
For the single-robot graphs $G^i$, we consider two constructions. 
The \textit{grid graph} (Fig.~\ref{fig:gridgraph}), which is task-agnostic, samples a spatial lattice in $(x,y)$ (with $\Delta x{=}\Delta y{=}1\,\mathrm{m}$) and expands said lattice with headings, speeds, and steering angles on discrete sets $(\theta,v,\delta)$. 
From each node, we attempt directed connections to \emph{nearby} nodes within a $k$-hop neighborhood of the lattice. 
The \textit{track graph} leverages the environment's topology and samples along a track centerline with lateral offsets (\appfigref{app:figures}{fig:trackgraph}). Unless otherwise stated, robots share identical dynamics and graph $G^i$.


Graph size is determined by the sampling resolutions $(\Delta x,\Delta y,\Delta\theta,\Delta v,\Delta\delta)$ and the connection parameter $k$. 
We choose these to meet a computational budget while preserving sufficient fidelity for the maneuvers of interest. 
Concretely, we set spatial and angular resolutions and $k$ to reach a desired ballpark of nodes/edges, then prune infeasible nodes/edges via kinodynamic feasibility and obstacle inflation. 
The resulting graph sizes are reported in Table~\ref{tab:map_solve_time_full_compact}.


Edge generation is performed in the following manner.
For each candidate source\(\rightarrow\)target pair of nodes, we solve a fixed-duration two-point boundary-value problem (BVP) over \(\Delta t=1\,\mathrm{s}\)
with hard input limits \(|a|\le 5\,\mathrm{m/s^2}\), \(|\omega|\le 2\,\mathrm{rad/s}\).
If the BVP is feasible, we add the corresponding directed edge and store its continuous trajectory.
Edges are discarded if their continuous trajectory intersects obstacles inflated by \(r_{\mathrm{coll}}\).
Unless otherwise specified, each edge has unit cost.

\niceparagraph{Search heuristics.}
As the inner search \algisne may be called on any $\pi_{\xcand}$, we require an admissable heuristic to every node $x$, and not just to $\Xgoal$, thus for each robot $i$ we precompute all-pairs shortest-path (APSP) distances on $G^i$ and use
$H(x)=\sum_i H^i(x^i)$  (or the weighted equivalent when $J(\cdot)$ is weighted, see below) where $H^i$ is the single-robot shortest distance from $x^i$ to $\mathcal{X}^i_{\mathrm{goal}}$.
Because single‑robot distances satisfy the triangle inequality and joint costs add across robots, $H$ is consistent (and hence admissible) on the tensor‑product graph $G$. Moreover, during inner A* searches for robot $i$ on $G^i$ with dynamic obstacles induced by a fixed plan $\pi^{-i}$, $H^i$ remains a valid lower bound on $i$'s cost‑to‑go: the added constraints only restrict feasible motion and cannot reduce the optimal single‑robot cost. 

\subsection{Qualitative analysis} \label{subsec:Experiments}
We present a set of multi-agent scenarios mimicking interactions in autonomous driving and racing and discuss the results obtained by our algorithm. These experiments showcase the generality of our approach, and demonstrate the diversity of equilibria available, which are hard to capture with previous optimization-based methods. 


Across tasks we vary two interpretable knobs that shape the resulting NE: The global objective aggregates the individual costs $J(\pi_n)=\sum_{i=1}^m \alpha_i\,J^i(\pi^i_n,\pi^{-i}_n)$ to encode priorities via the weights $\alpha_i \in [0,1]$. Additionally, each $J^i$ includes a state-based penalty that encourages separation from other robots: 
{\small
    \[
    c^i_{\mathrm{prox}}\!\left(\mathbf{x}^i,\mathbf{x}^{-i}\right)
    \;=\;
    \frac{\lambda_{\mathrm{prox}}}{\max\!\left(\min_{j\in -i}\,\bigl\|\textup{proj}(\mathbf{x}^i)-\textup{proj}(\mathbf{x}^j)\bigr\|_2,\ \varepsilon\right)},
    \]
}
    where $\lambda_{\mathrm{prox}}\!\ge 0$, $\textup{proj}:\mathcal{X}\to\mathbb{R}^2$ returns the planar position $(x,y)$ of a state and $\varepsilon{=}10^{-3}$ prevents singularities (yielding a bounded, well-behaved, Lipschitz stage cost). Larger $\lambda_{\mathrm{prox}}$ promotes greater standoff. In our experiments, each $J^i$ is the sum over $n$ edges of $\pi^i_n$ of unit edge cost plus the proximity term for each of the $n+1$ nodes along the way:
{\small    
\[
J^i(\pi^i_n,\pi^{-i}_n) = n +
\sum_{\mathbf{x}^i, \mathbf{x}^{-i}\in(\pi^i_n,\pi^{-i}_n)}\!\Bigl(c^i_{\mathrm{prox}}(\mathbf{x}^i,\mathbf{x}^{-i})\Bigr).
\]
}


Next, we discuss the scenarios and the solution obtained by \algd. An additional experiment that showcases the effect of $\lambda_{\mathrm{prox}}$ on following distance in a semi-collaborative setting, can be found in \appfigref{app:figures}{fig:racedistance}.

\begin{wrapfigure}[9]{r}{0.3\columnwidth}\vspace{-1.6\baselineskip}  
    \centering
    \begin{overpic}[
        width=\linewidth,
        trim=140pt 100pt 2pt 50pt,
        clip
    ]{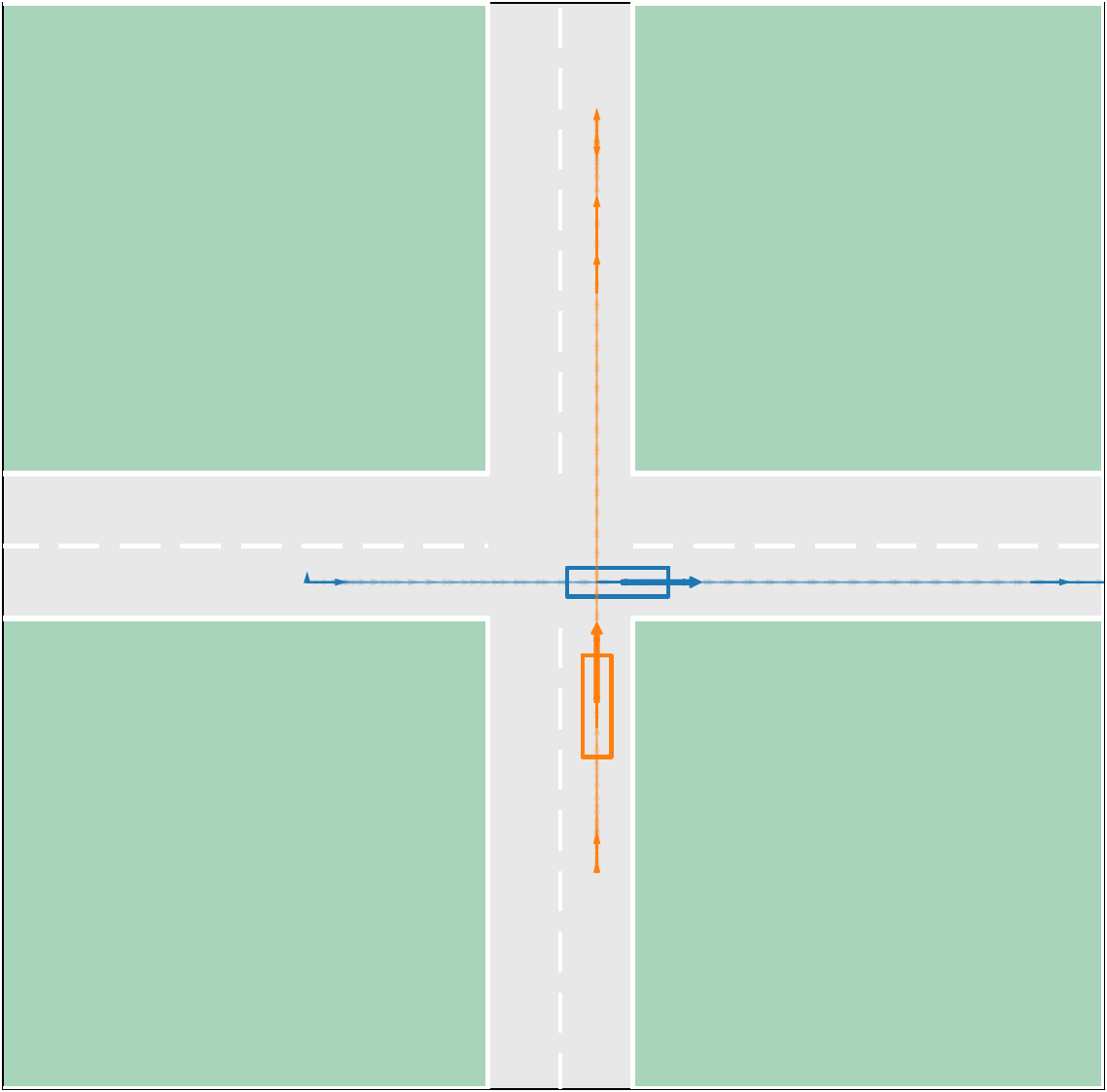}
    
        \put(65,72){%
            \colorbox{white}{\scriptsize $\alpha \geq 0.5$}%
        }
    \end{overpic}
    \vspace{-100pt}
\end{wrapfigure}
\niceparagraph{Four-way intersection.} We consider a single-lane four-way intersection with perpendicular approaches by two robots, objective $J=\alpha J^1+(1-\alpha)J^2$, $\alpha\in[0,1]$, $\lambda_{\mathrm{prox}}=0$.
For $\alpha \geq 0.5$ the NE selected by our planner has Robot~1 (blue) crossing from left to right while Robot~2 (orange) yields and waits for the intersection to clear before proceeding from bottom to top. For $\alpha < 0.5$, the NE is symmetric: Robot~2 crosses first and Robot~1 yields. This result illustrates the ability of our NE planner to select a safe NE in an explainable manner (via tuning of the objective function $J$), without requiring clipping of the dynamics or costs, or initial guesses (as commonly considered in previous work). 


\niceparagraph{Racetrack overtake.} We consider a racetrack segment to showcase priority-driven overtaking decisions between two robots. This setting  probes how a scalar priority $\alpha$ biases raceline choice and pass timing within realistic behavioral dynamics trajectories. 
We set $J=\alpha J^1+(1-\alpha)J^2$, $\alpha\in[0,1]$, $\lambda_{\mathrm{prox}}=0$. Fig.~\ref{fig:raceovertake} shows that the tilt of robot priority determines which robot targets the inside line at the final turn, and ultimately wins the race. Note that even at high $\alpha$, Robot 1 is blocked at the early corner---illustrating that even in the best NE for Robot 1, overtaking early is not feasible due to kinodynamic/safety constraints. Thus, via a  tuned NE our algorithm can robustly express the lower bound on the cost of a given robot's solution (within the space of all NEs). 

\begin{figure}[t]
  \centering
  \begin{overpic}[width=0.8\textwidth]{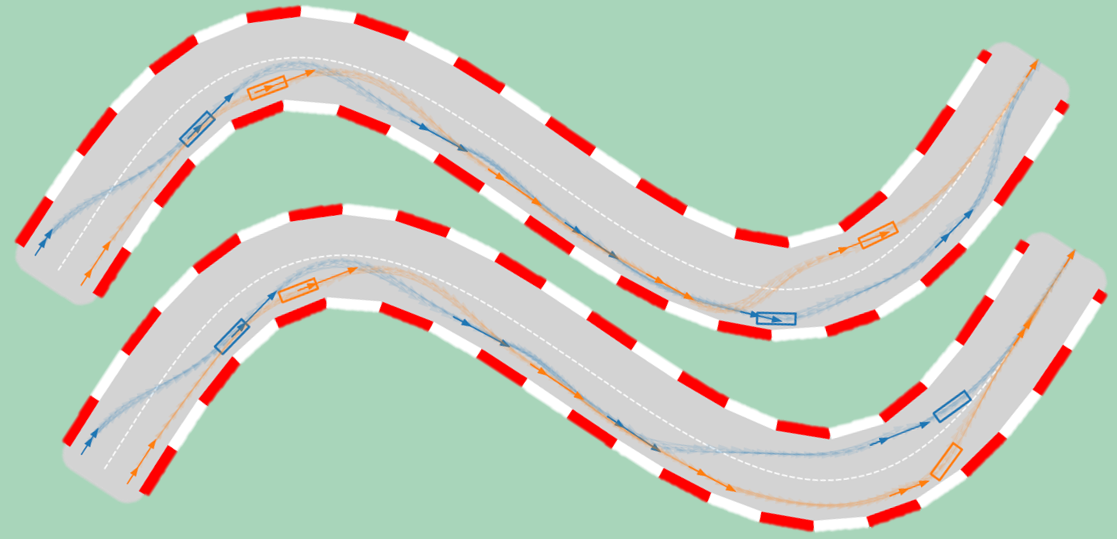}
    \put(63,37){\colorbox{white}{\scriptsize $\alpha < 0.5$}}
    \put(26,11){\colorbox{white}{\scriptsize $\alpha \geq 0.5$}}
  \end{overpic}
  \caption{Racetrack overtake. Prioritized robot takes inside lane at final turn and wins the race. Robot 1 (blue) always blocked at first turn by Robot 2 (orange) despite priority.}
  \label{fig:raceovertake} 
\end{figure}


\niceparagraph{Multi-lane highway merge.} We consider a three-lane highway merge with three robots (Fig.~\ref{fig:highwaymerge}) to  explore homotopy classes of merge options as $\alpha$ and $\lambda_{\mathrm{prox}}$ trade priority against risk.  
We set  $J=(1-\alpha)J^1+\alpha(J^2+J^3)$, $\alpha\in[0,1]$, and $\lambda_{\mathrm{prox}}\in[0,2.5]$. 
As priority for Robot 1 decreases and $\lambda_{\mathrm{prox}}$ increases, the NE transitions through distinct maneuvers. 
This shows the multiplicity of homotopy classes that often (if not always) exist in complex interactions, that are often impossible to capture with local NE solvers. 

\niceparagraph{Opposing-lane overtake.}
We consider a two-lane bidirectional road with oncoming traffic and three robots to simulate risk-aware passing with heterogeneous capabilities ($G^1$ has $v\in[0,10],\ k=5$, $G^2, G^3$ have $v\in[0,4],\ k=3$). We also test how $\alpha$ and $\lambda_{\mathrm{prox}}$ modulate dangerous overtake decisions and select among oncoming-lane vs. in-lane homotopy classes. We set  $J=(1-\alpha)J^1+\alpha(J^2+J^3)$, $\alpha\in[0,1]$,  and $\lambda_{\mathrm{prox}}\in[0,1]$. In Fig.~\ref{fig:roadovertake} we see that with high priority and low $\lambda_{\mathrm{prox}}$, Robot 1 overtakes at the first viable opportunity. Increasing $\lambda_{\mathrm{prox}}$ and lowering priority, however, delays the pass until oncoming traffic clears, and beyond a certain threshold, the NE forgoes the overtake entirely. This  shows behavioral changes across the different homotopy classes. 





\begin{figure}[b] 
  \centering
  \captionsetup{aboveskip=0pt, belowskip=0pt}

  \begin{subfigure}{\columnwidth}
    \centering
    \begin{overpic}[%
        width=0.8\linewidth,
        trim=3pt 5pt 3pt 5pt, 
        clip
    ]{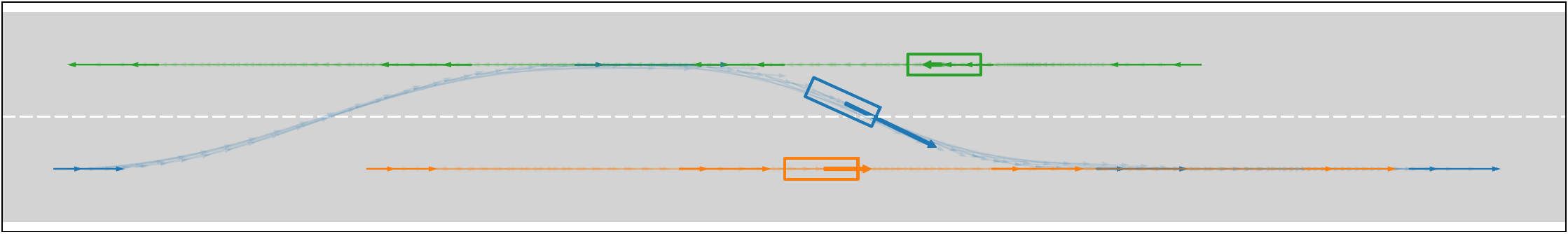}
      \put(77,12){\fcolorbox{gray}{white}{\scriptsize (a) $\alpha = 0.00,\ \lambda_{\mathrm{prox}} = 0.00$}}
    \end{overpic}
    \phantomsubcaption\label{fig:highwayovertake000}
  \end{subfigure}\vspace{-0.7em}

  \begin{subfigure}{\columnwidth}
    \centering
    \begin{overpic}[%
        width=0.8\linewidth,
        trim=3pt 5pt 3pt 5pt,
        clip
    ]{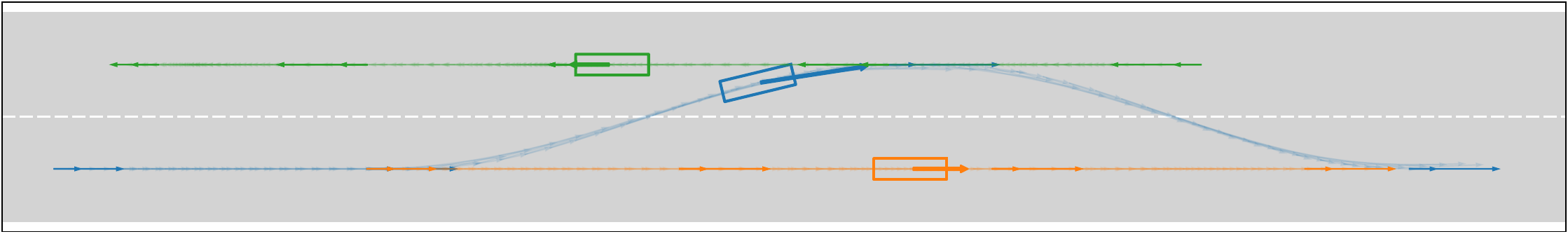}
      \put(77,12){\fcolorbox{gray}{white}{\scriptsize (b) $\alpha = 0.25,\ \lambda_{\mathrm{prox}} = 0.25$}}
    \end{overpic}
    \phantomsubcaption\label{fig:highwayovertake025}
  \end{subfigure}\vspace{-0.7em}

  \begin{subfigure}{\columnwidth}
    \centering
    \begin{overpic}[%
        width=0.8\linewidth,
        trim=3pt 5pt 3pt 5pt,
        clip
    ]{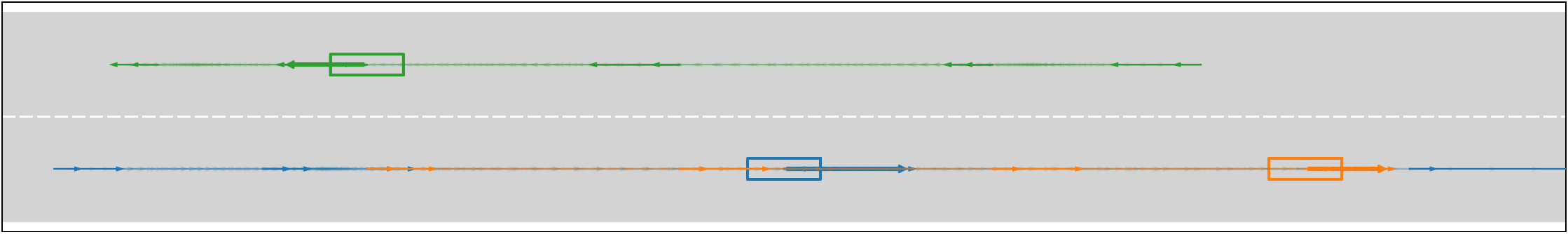}
      \put(77,12){\fcolorbox{gray}{white}{\scriptsize (c) $\alpha \ge 0.50,\ \lambda_{\mathrm{prox}} \ge 0.50$}}
    \end{overpic}
    \phantomsubcaption\label{fig:highwayovertake05+}
  \end{subfigure}

  \caption{Opposing-lane overtake. Robot 1 (blue) tries to overtake Robot 2 (orange) via Robot 3's (green) opposing lane. Different NE overtake strategies chosen as priority and proximity penalty are modulated.}
  \label{fig:roadovertake}
\end{figure}

\subsection{Running time evalution}

\begin{table}[h!]
\centering
\small
\setlength{\tabcolsep}{2pt}
\sisetup{
  detect-weight = true,
  group-minimum-digits = 4,
  group-separator = {,},
  table-number-alignment = center,
  table-text-alignment = center,
}
\caption{Running time evaluation for \algd. Each row corresponds to a map, number of robots ($m$). The rightmost block lists time in seconds with the associated goal distance $D$ in parentheses. The \emph{Generic Free} map is an obstacle-free workspace used to stress test expansion (maximal branching). \emph{Generic Obstacles} uses the same spatial dimensions but has random obstacles, pruning many nodes/edges (lower branching).} 

\begin{tabular}{l
                S[table-format=4]
                S[table-format=6]
                S[table-format=1]
                *{6}{c}}
\toprule
\textbf{Map} & $\mathbf{|V|}$ & $\mathbf{|E|}$ & $\mathbf{m}$ &
\multicolumn{6}{c}{$\mathbf{\textbf{Solve Times } (sec) \textbf{ by Distance } D}$} \\
\midrule
Generic      & 7560 & 51733 & 2 & 3.31(7) & 0.046(9) & 2.11(11) & 1.57(13) & 0.063(15) & 4.73(17) \\
Free                                   &  &  & 3 & 396.1(7) & 3.7(9) & 59.1(11) & 16.2(13) & 9(15) & 178.5(17) \\
\midrule
Generic  & 6300 & 25818 & 2 & 0.051(7) & 0.039(9) & 1.28(11) & 0.666(13) & 0.797(15) & 0.136(17) \\
Obstacles                                   &  &  & 3 & 15.1(7) & 34(9) & 3.79(11) & 12.2(13) & 4.58(15) & 12(17) \\
\midrule
Two-Lane  & 5138 & 39618 & 2 & 0.034(7) & 1.02(9) & 0.984(11) & 0.501(13) & 4.59(15) & 2.77(17) \\
Road       &   &  & 3 & 10.7(7) &  &  &  &  &  \\
\midrule
Track            & 693  & 4359  & 2 & 0.097(7) & 0.24(9) & 0.435(11) & 0.865(13) & 0.247(15) & 1.55(17) \\
                                 &   &   & 3 & 60.3(7) & 2.22(9) & 32.1(11) & 55.2(13) & 525.9(15) &  \\
\midrule
Crossroads       & 1428 & 550   & 2 & 0.005(2) & 0.006(4) & 0.004(6) & 0.007(7) & 0.007(8) &  \\
                                 &  &    & 3 & 0.046(2) & 0.05(4) & 0.045(6) & 0.037(7) & 0.041(8) &  \\
                                 &  &    & 4 & 0.162(2) & 0.111(4) & 0.154(6) & 1.53(7) & 0.283(8) &  \\
                                 &  &    & 5 & 0.724(2) & 1(4) & 12.2(6) & 4.93(7) & 12.1(8) &  \\
\bottomrule
\end{tabular}
\label{tab:map_solve_time_full_compact}
\end{table}
Table~\ref{tab:map_solve_time_full_compact} reports wall-clock solve times (not including preprocess time for graph generation) for \algd\ across several maps, goal distances and number of robots. Goal distance $D$ is a proxy for planning horizon $T$ as a time cap $T$ restricts search to the subgraph reachable within $\lfloor T/\Delta t \rfloor$ edges from the start; by instead varying the goal distance $D$, we probe increasingly far targets without imposing this cap.

We note that increasing $D$ does not substantially increase running time---a considerable strength of our approach. Unlike optimization-based approaches---where running time typically scales with horizon due to the growing number of decision variables---our graph-based formulation keeps per-step local problems fixed-size and searches over a discrete action set, so increasing $D$ does not balloon the optimizer. By contrast, running time generally increases with the number of robots. The joint neighbor set grows combinatorially, interaction geometry becomes richer (more constraints and blocking situations), and the equilibrium check \algisne\ triggers more single-robot A* calls (a principal computational bottleneck) as $m$ grows. An exception is the \emph{Crossroads} map: its branching factor is small and interactions are simple, so adding robots does not markedly increase search effort.

Across all settings, \emph{instance difficulty} (map layout and interaction geometry) dominates running time. 
Tight merges, opposing-lane passes, and narrow corridors can make some two-robot cases slower---or even infeasible at a given discrete horizon---than multi-robot cases on simpler maps with few viable options. 
This can also explain why \emph{Generic Obstacles} is faster on average than \emph{Generic Free}: moderate obstacle density reduces branching without making the graph too sparse. See \appfigref{app:figures}{fig:genericmaps} for visualization.

\subsection{Comparison to LaValle-Hutchinson method}
As noted in Sec.~\ref{sec:Intro}, the LaValle-Hutchinson method~\cite{lavalle2002optimal} also leverages graph search techniques and yields equilibrium guarantees. Here, we study the computational burden of generating, pairwise verifying, and explicitly maintaining all non-dominated candidate trajectories (up to one representative per equivalence class) within that approach. Unsurprisingly, in maps with low branching, and short horizons (e.g., Crossroads, Track), and two robots, the two algorithms are almost comparable: our method achieves a modest speedup of roughly $1.5$-$2\times$. For dense graphs with high branching factors and long horizons (e.g., Two-Lane-Road, Generic Free), however, the speedup reaches $2$-$3$ \emph{orders of magnitude}.

The main reason is that the LaValle-Hutchinson algorithm must carry \emph{every} non-dominated partial trajectory (equivalency class) at each joint state. As the branching factor and horizon grow, the number of distinct cost vectors at a single node can grow combinatorially, so each additional layer multiplies both the number of wavefront paths and the cost of dominance checks. In contrast, our planner explores the composite state space with a best-first strategy guided by single-robot distance heuristics and aggressive dominance, so it (a) explores more promising areas first (only) and (b) effectively prunes entire sub-trees. Thus, its complexity scales primarily with the number of joint states actually visited, rather than with the total number of distinct joint trajectories, which makes high-branching and large-horizon regimes much easier to handle in practice.

\vspace{-5pt}
\section{Conclusion and future work}\label{sec:conclusion} We presented \algdfull, a scalable framework for computing global Nash equilibria for game-theoretic motion planning with general dynamics. \algd enjoys strong theoretical guarantees, while also allowing explicit selection among multiple equilibria via a user-specified global objective. Across driving and racing scenarios, GTNS finds collision-free, dynamically feasible interactions within seconds.

In the future, we plan to address a key limitation of the approach, which is the reliance on high-fidelity kinodynamic roadmaps, whose construction requires many calls to a BVP solver; although peformed offline, and can be reused in different scenarios, the compute cost and memory footprint grow with state discretization and connection density, and are a key practical bottleneck (e.g., map construction for the overtake scenario (Fig.~\ref{fig:roadovertake}) required several hours to compute). One possible direction could be substituting such graphs with kinodynamic trees that are computed on-the-fly by sampling-based planners~\cite{li2016asymptotically,HauserZ16,FuSSA23}. This modification should be done with care as to ensure that the game-theoretic guarantees hold. 
We also plan to explore improving the algorithm's scalability with respect to the number of robots. Although the tensor graph is explored implicitly, joint branching still grows with the number of robots; moreover, each equilibrium check triggers single-robot searches that can dominate running time in tightly coupled scenes, as well as parallelizing the single-robot searches.

In the longer run, we plan to extend the approach to the closed-loop setting and study principled tie-breaking and fairness- and risk-aware objectives, as well as coordination mechanisms when several NEs exist.

\subsubsection{Acknowledgments.} 
This work was supported by the Israel Ministry of Science and Technology grant no.\ 2034500 and the Technion Autonomous Systems Program grant no.\ 2072638. The AI system ChatGPT was used for light editing and grammar enhancement, as well as a preliminary literature review.

\bibliographystyle{splncs04}
\bibliography{references}

\ifwafr
\else
  \appendix
\renewcommand{\theHsection}{\Alph{section}}

\section{Implementation Speedups}\label{app:speedup}
As mentioned in Sec.~\ref{sec:Algorithm}, the pseudocode is presented so that the fundamentals of \algd and \algisne are as clear as possible. In practice, we obtain significant efficiency improvements by leveraging several implementation techniques. 

The most consequential (inspired by works such as \cite{dellin2016unifying, mandalika2019generalized, cohen2015planning}) is \emph{lazy validation} of collision and equilibrium checks: every neighbour of $x$ visited in line~\ref{alg:ExtendNodeLine} is pushed to the \texttt{OPEN} list without immediate validation; only when a neighbour is popped from \texttt{OPEN} (i.e., it currently has the lowest cost) do we verify that it is in fact a valid node. This defers the most expensive computations and spends resources only on promising candidates.

The implementation also caches edge-collision results (Boolean flags indicating whether an edge $e \in E$ is in collision) to allow fast lookup before invoking an expensive check, precomputes heuristics, and performs early feasibility checks for the query. Additional optimizations include careful memory management and the use of pointers to stable storage. Finally, we note that \algisne is trivially parallelizable, as the per-robot A* runs with respect to fixed trajectories of all other robots are independent.

\section{Additional Experiments \& Figures}\label{app:figures}
This appendix includes supplementary figures and experiments that exceed the space limitations of the main manuscript. Specifically, we examine the effects of NE mismatch and its resulting adverse outcomes. Additionally, we present an experiment that isolates the influence of $\lambda_{\mathrm{prox}}$ in a semi-collaborative setting. Finally, a summary of the notation used in the paper is provided.

\subsection{Effect of NE mismatch} \label{app:mismatch}
Our approach allows for computing various NE solutions to the same problem.  The experimental results indicate that different equilibria can result in vastly distinct performance of the multi-robot system. So far, we have assumed that all the agents execute the same equilibrium, which may not be guaranteed in practice. To motivate further study in coordinating the equilibrium between agents, we evaluate the consequences of equilibrium mismatch. 

In the majority of tested cases, when agents executed in an open-loop fashion trajectories belonging to different NEs,  collisions between robots were encountered. This underscores the need for a coordinated, team-wide equilibrium in behavior-aware motion planning. To address the concern that open-loop per-agent NEs are ``obviously brittle,'' we also evaluated an MPC variant: at every step each agent re-solved \emph{its own} NE instance, where $J=J_i$, i.e., the global cost strictly optimizes the solution of the specific robot executing its trajectory from this NE.  At every step, the robot resolved the problem towards the goal using this objective while observing the updated locations of the other robots, executed its step, and repeated the process. Such an approach is hoped to be more resilient as it could provide more opportunities for the robots to avoid colliding with each other. Unfortunately, even under this closed-loop regime, the lack of shared equilibrium selection caused systematic issues and in almost all scenarios (e.g., those in Figs.~\ref{fig:roadovertake} and \ref{fig:highwaymerge}) a collision. In rare ``best-case" scenarios,  when interaction geometry was simple, (e.g., on the track in Fig.~\ref{fig:racedistance} with large start distance) mismatched individual NEs yielded similar solutions to those returned by the global planner. 

We believe that such a NE mismatch could explain the cause of motor accidents between human drivers where a driver ``fails to judge other driver’s path or speed'', which contributes to at least \%13 of all accidents in the US, according to a 2013 report~\cite{Najm2013PreCrashLightVehicle}, and \%11 of fatal accidents in the UK, based on a 2022 report~\cite{DfT2024RSFInitialAnalysis }. This suggests that algorithms such as \algd could be used by transportation engineers to assess the impact, in terms of safety, of introducing a new road feature (e.g., an intersection or a traffic island).

\begin{figure}
  \centering
  \begin{overpic}[width=\textwidth]{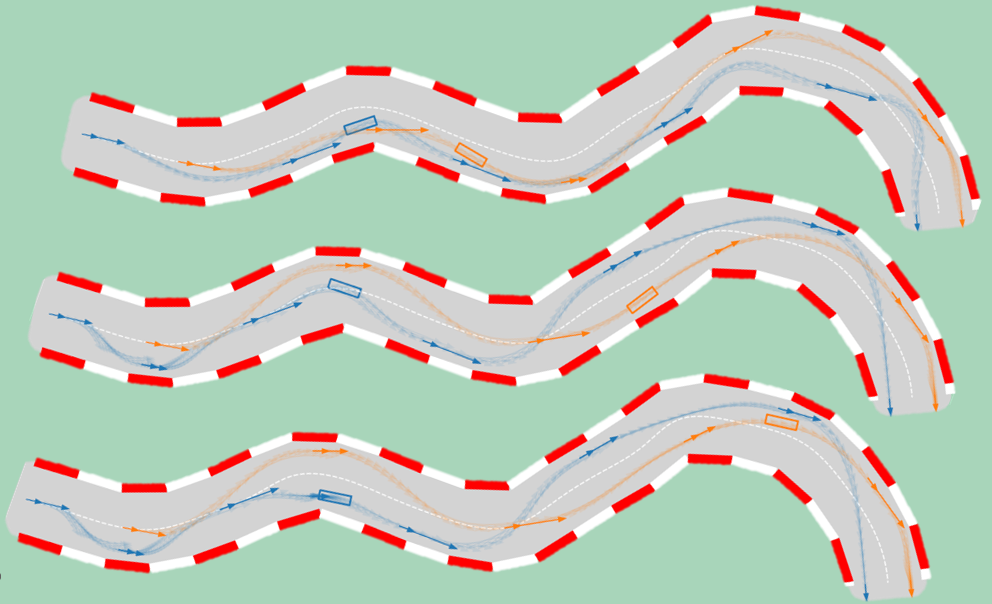}
    \put(2,53){\colorbox{white}{\scriptsize $\lambda_{\mathrm{prox}}=0.00$}}
    \put(2,35){\colorbox{white}{\scriptsize $\lambda_{\mathrm{prox}}=0.25$}}
    \put(2,16){\colorbox{white}{\scriptsize $\lambda_{\mathrm{prox}}\geq0.50$}}
  \end{overpic}
  \caption{Following distance. Semi-collaborative setting, such as two cars on the same team, with social welfare cost $J=\sum_i J^i$ and $\lambda_{\mathrm{prox}}\in[0,1]$. The cars are attempting to cross the finish line, with the lowest possible summed cost. The addition of a proximity penalty to the cost, affects the following distance that Robot 1 (blue) maintains from Robot 2 (orange): A small $\lambda_{\mathrm{prox}}$ value yields close drafting, and larger values increase the time-gap as Robot 1 deliberately backs off from Robot 2, visible at the first corner.}
  \label{fig:racedistance}
\end{figure}
\begin{figure}[h!]
  \centering
  \captionsetup{aboveskip=2pt, belowskip=0pt}
  \captionsetup[subfigure]{skip=2pt, belowskip=2pt} 

  \begin{subfigure}[t]{\textwidth}
    \centering
    \includegraphics[width=\textwidth]{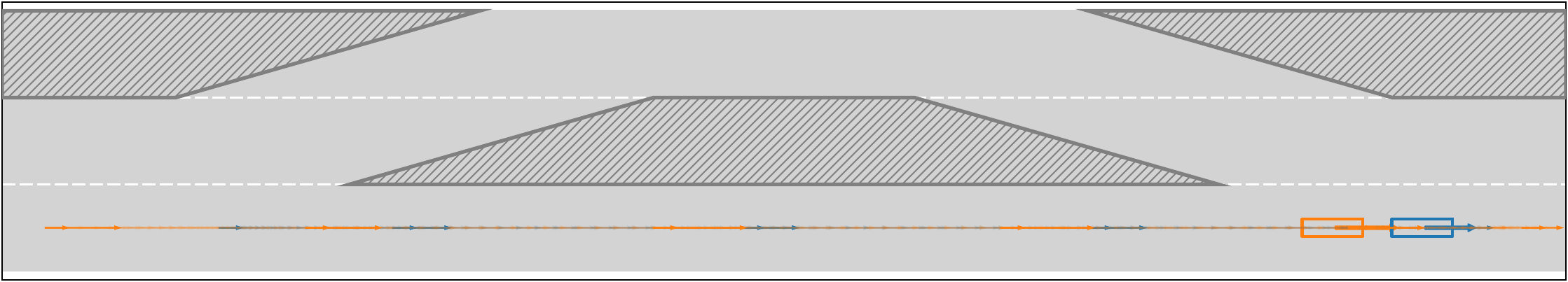}
    \subcaption{Local NE Solver}
    \label{fig:local}
  \end{subfigure}


  \begin{subfigure}[t]{\textwidth}
    \centering
    \includegraphics[width=\textwidth]{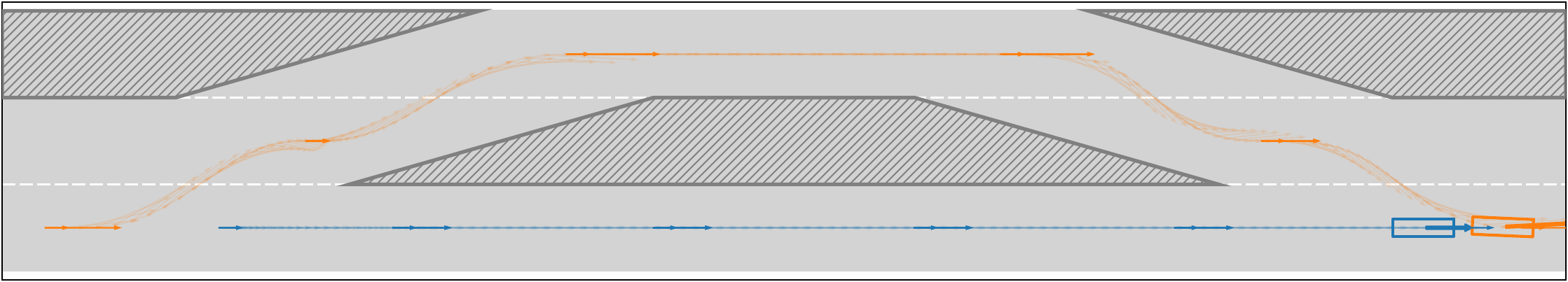}
    \subcaption{Global NE Solver}
    \label{fig:global}
  \end{subfigure}

  \caption{Local vs.\ global Nash equilibrium (NE) in a narrow-overtake scenario, where Robot 1 (blue) is attempting to overtake Robot 2 (orange).
    (a) A local, initialization-dependent solver converges to a non-overtake NE within the homotopy class implied by its seed. 
    (b) Our global NE method searches across homotopy classes and discovers here the overtaking NE.}
  \label{fig:localvsglobal}
\end{figure}

\begin{figure}[h!]
  \centering
  \includegraphics[width=0.98\columnwidth]{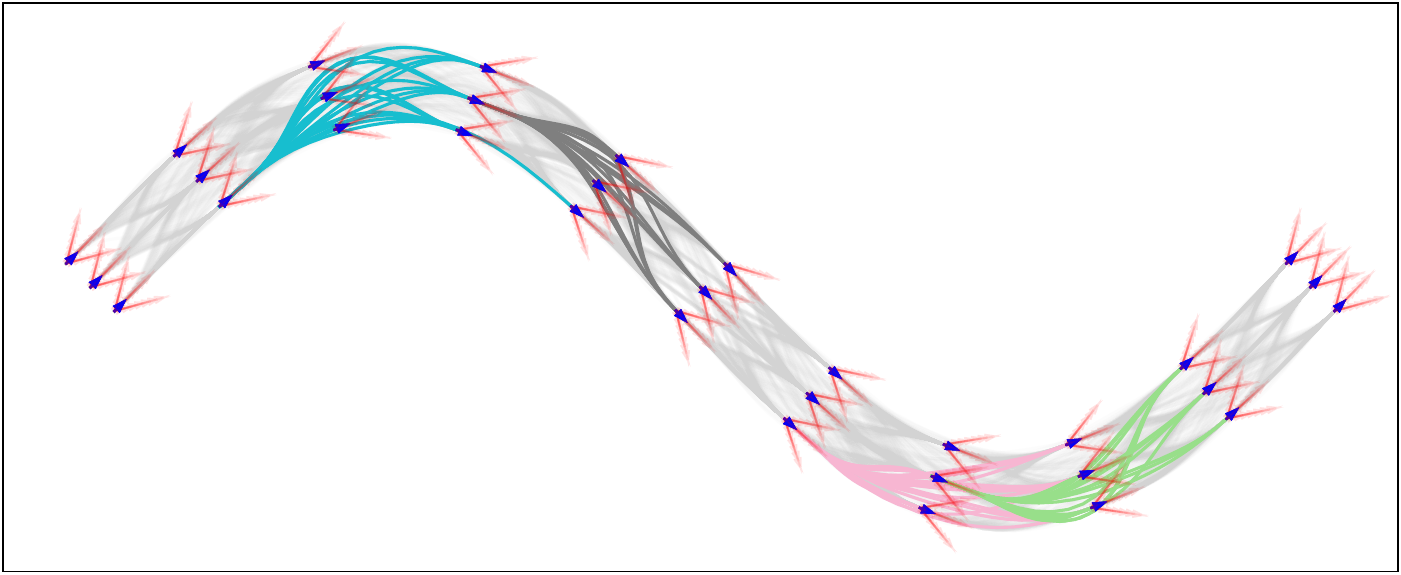}
  \caption{Kinodynamic \emph{track graph} built from racetrack centerline with lateral offsets. At each longitudinal sample we define a wayline (normal to the centerline), set the heading from the centerline tangent, and discretize $(v,\delta)$. From each wayline $\ell$ we attempt connections from all nodes on $\ell$ to all nodes on any wayline $\ell'$ with $|\ell'-\ell|\le k$. Each discrete state $(x,y,\theta,v,\delta)$ is drawn with two arrows anchored at $(x,y)$: a blue heading arrow pointing in direction $\theta$ (fixed short length), and a red steering arrow pointing in direction $\theta+\delta$ with length $\propto v$. Edge styling as in grid-graph.}
  \label{fig:trackgraph}
\end{figure}
\begin{figure}[h!]
  \centering
  \includegraphics[width=0.7\columnwidth]{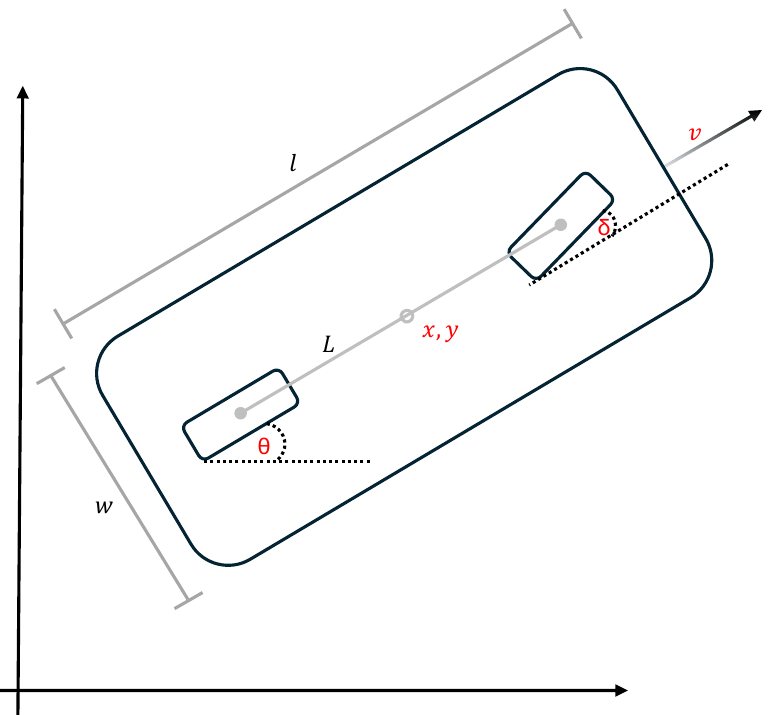}
  \caption{Second-order bicycle model for car dynamics. Car dimensions $l, w, L$ in black; State variables $(x,y,\theta,v,\delta)$ in red.
  In our experiments, we used a scale-model car of length $l=0.70\,\mathrm{m}$, width $w=0.20\,\mathrm{m}$, and wheelbase $L=0.50\,\mathrm{m}$.}
  \label{fig:secondordermodel}
\end{figure}
\begin{figure}[h!]
  \centering
  \begin{subfigure}[t]{\textwidth}
    \centering
    \includegraphics[width=\linewidth]{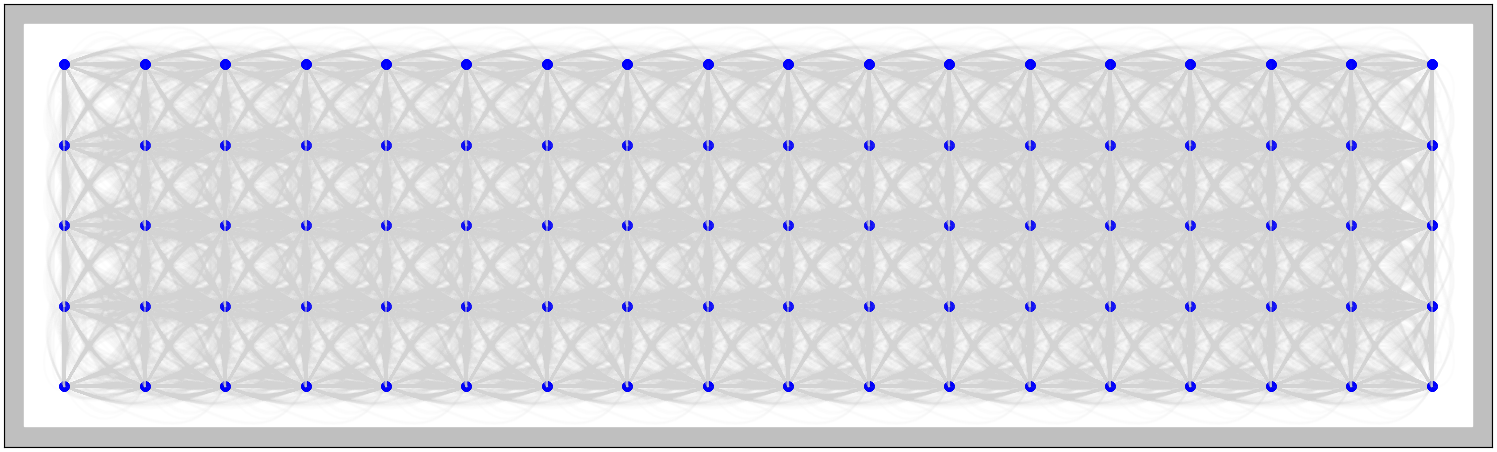}
  \end{subfigure}

  \vspace{-0.1em}

  \begin{subfigure}[t]{\textwidth}
    \centering
    \includegraphics[width=\linewidth]{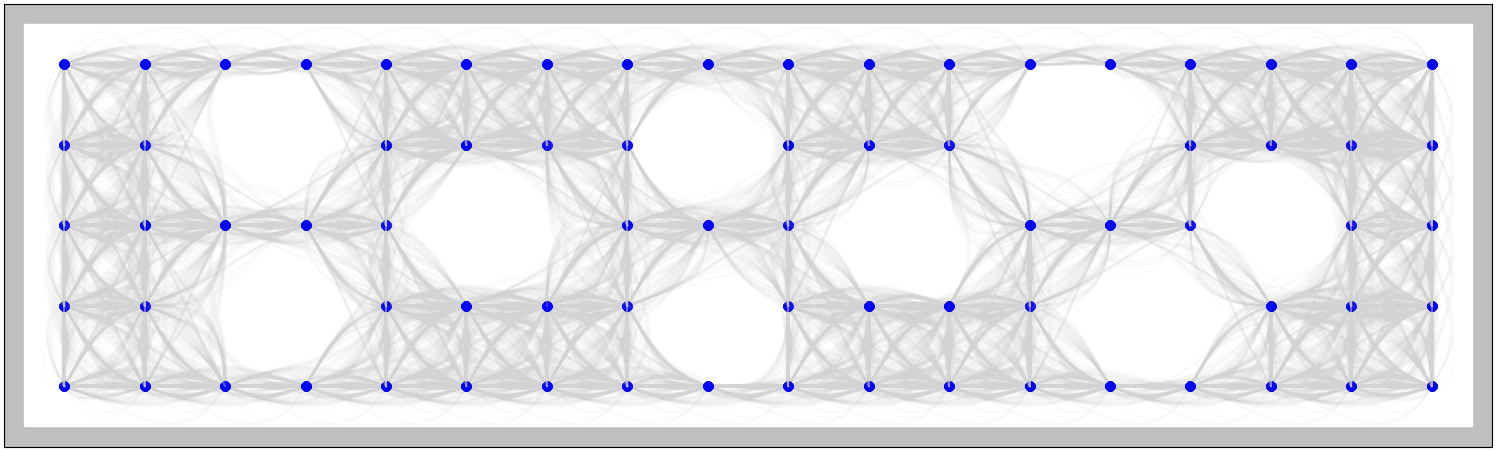}
  \end{subfigure}
  
  \vspace{-0.5em}
  \caption{Generic maps: with (bottom) and without (top) obstacles.}
  \label{fig:genericmaps}
\end{figure}
\begin{table}[h!]
\centering
\renewcommand{\arraystretch}{1.1}
\caption{Notation for continuous and graph-based multi-robot planning.}
\begin{tabular}{lll}
\hline
SR Symbol & MR Symbol & Meaning \\
\hline
& $[m]$ & Robot index set $\{1,...,m\}$ \\
$d_i, D_i$ &  & Dim.\ of $\mathcal{X}^i \subseteq \mathbb{R}^{d_i}$, $\mathcal{U}^i \subseteq \mathbb{R}^{D_i}$ \\
$\mathcal{X}^i$ & $\mathcal{X} = \times_{i=1}^m\mathcal{X}^i$ & Robot $i$ (joint) state space \\
$\mathcal{U}^i$ & $\mathcal{U}=\times_{i=1}^m\mathcal{U}^i$ & Robot $i$ (joint) control space \\
$\mathcal{X}^i_f \subset \mathcal{X}^i$ & $\mathcal{X}_f=\times_{i=1}^m\mathcal{X}_f^i \subset \mathcal{X}$ & Collision–free state space \\
$\mathcal{U}^i_T$ & $\mathcal{U}_T=\times_{i=1}^m \mathcal{U}^i_T$ & Admissible controls $u^i:[0,T]\mapsto\mathcal{U}^i$ \\
& & $\bigl(u:[0,T]\mapsto\mathcal{U}\bigr)$  \\
$x^i(t)$ & $x(t){=}(x^1(t),...,x^m(t))$ & Robot $i$ (joint) state at time $t\in[0,T]$ \\
$u^i(t)$ & $u(t){=}(u^1(t),...,u^m(t))$ & Robot $i$ (joint) control at time $t\in[0,T]$ \\
$x_0^i$ & $x_0$ &  Robot $i$ (joint) initial state \\
$\mathcal{X}^i_{\mathrm{goal}}$ & $\mathcal{X}_{\mathrm{goal}}=\times_{i=1}^m \mathcal{X}^i_{\mathrm{goal}}$ &   Robot $i$ (joint) goal region \\
$\pi^i_{x_0^i,u^i}$ & $\pi_{x_0,u}{=}(\pi^1_{x_0^1,u^1},...,\pi^m_{x_0^m,u^m})$ & Robot $i$ (joint) continuous trajectory \\
 &  & $\pi^i : [0,T]\mapsto \mathcal{X}^i$ $\bigl(\pi :[0,T]\mapsto \mathcal{X}\bigr)$ \\
$\Pi^i$ & $\Pi=\times_{i=1}^m\Pi^i$& All robot $i$ (joint) feasible continuous \\
& & trajectories \\
$c^i(\pi^i(t),\pi^{-i}(t))$ & & Stage cost of robot $i$ given its and \\
& & others’ trajectories \\
& $c(\pi(t))$ & Joint stage cost over joint trajectory \\
$J^i(\pi^i,\pi^{-i})$ & & Cost of robot $i$ given its and \\
& & others’ trajectories \\
& $J(\pi)$ & Global cost over joint trajectory \\
$f^i(x^i,u^i)$ & $f(x,u)$ & Robot $i$ (joint) dynamics: $\dot x^i(t)=$\\
& & $f^i(x^i(t),u^i(t))$ $\bigl(\dot x(t)=f(x(t),u(t))\bigr)$ \\
\hline
$\Delta t$ & & Fixed duration of each edge \\
$G^i=(V^i,E^i)$ & $G=(V,E)$ & Robot $i$ (joint, tensor-product) graph \\
$V^i$ & $V=\times_{i=1}^m V^i$ & Robot $i$ (joint) sampled states $x^i\in \mathcal{X}^i$ \\
& &incl.\ $x_0^i$, $\mathcal{X}^i_{\mathrm{goal}}$ $\bigl(x\in \mathcal{X}$ incl. $x_0$, $\mathcal{X}_{\mathrm{goal}}\bigr)$ \\
$E^i$ & $E=\times_{i=1}^m E^i$ & Robot $i$ (joint) edges $e^i = (x^i,y^i)$ \\
& & $\bigl(e = (x,y)\bigr)$ with control $u^i_{e^i} \in \mathcal{U}^i_{\Delta t}$ \\
& & $\bigl(u_{e} \in \mathcal{U}_{\Delta t}\bigr)$ \\
$\pi^i_{e^i}$ & $\pi_e{=}(\pi^1_{e^1},...,\pi^m_{e^m})$ & Robot $i$ (joint) trajectory on edge $e^i\ \bigl(e\bigr)$ \\
$\pi^i_n{=}(x^i_{e^i_1},...,x^i_{e^i_n})$ & $\pi_n{=}(x_{e_1},...,x_{e_n})$ & Robot $i$ (joint) $n$-step trajectory \\
& &  on $G^i$ $\bigl(G\bigr)$ \\
$\pi^i_{x^i}$ & $\pi_x{=}(\pi^1_{x^1},...,\pi^m_{x^m})$ & Robot $i$ (joint) trajectory to state $x^i\ \bigl(x\bigr)$ \\
$\pi^i_{x^i\to y^i}$ & $\pi_{x\to y}$ & Robot $i$ (joint) local trajectory from \\
& & state $x^i$ to $y^i\ \bigl(x$ to $y\bigr)$ \\
$\Pi^i_n$ & $\Pi_n=\times_{i=1}^m\Pi^i_n$ & Robot $i$ (joint) feasible $n$-step \\
& & trajectories \\
$P^i(x^i)$ & $P(x)$ & Predecessor map of node $x^i\ \bigl(x\bigr)$ \\
\hline
& $\varepsilon$ & Approximation factor for $\varepsilon$-equilibrium \\
\hline
\end{tabular}
\label{tab:nomenclature}
\end{table}

\fi

\end{document}